\documentclass{article}
\usepackage[final]{neurips_2021}
\usepackage{booktabs} 
\usepackage[ruled]{algorithm2e} 

\SetAlFnt{\small}
\SetAlCapFnt{\small}
\SetAlCapNameFnt{\small}
\SetAlCapHSkip{0pt}
\IncMargin{-\parindent}

\author{%
Etienne Boursier\\
Centre Borelli, ENS Paris-Saclay, France\\
\texttt{etienne.boursier1@gmail.com}
	\AND
  Tristan Garrec\\
  Centre Borelli, ENS Paris-Saclay, France \\
  EDF Lab, Palaiseau, France\\
  \texttt{tristan.garrec@ut-capitole.fr} 
   \And
   Vianney Perchet \\
  CREST, ENSAE Paris, France \\ Criteo AI Lab, Paris, France \\
   \texttt{vianney@ensae.fr}
   \AND
   Marco Scarsini \\
  Department of Economics and Finance, LUISS, Rome \\
   \texttt{marco.scarsini@luiss.it}
}

\usepackage[utf8]{inputenc} 
\usepackage[T1]{fontenc}    
\usepackage{hyperref}       
\usepackage{url}            
\usepackage{booktabs}       
\usepackage{amsfonts}       
\usepackage{nicefrac}       
\usepackage{microtype}      
\usepackage{wrapfig}
\usepackage{changepage}

\usepackage{amsmath}
\usepackage{amssymb}
\usepackage{amsthm}
\usepackage{mathtools}
\usepackage{caption}
\usepackage{subcaption}
\usepackage{longtable}
\usepackage{float}
\allowdisplaybreaks

\usepackage{etoolbox}
\appto\normalsize{\belowdisplayskip=\belowdisplayshortskip}
\appto\small{\belowdisplayskip=\belowdisplayshortskip}
\usepackage{dsfont}

\numberwithin{equation}{section}  
\usepackage[sort&compress,capitalize,nameinlink, noabbrev]{cleveref}
\Crefname{app}{Appendix}{Appendices}
\crefname{algorithm}{Algorithm}{Algorithms}

\newcommand{\ie}{i.e.,\ }
\newcommand{\eg}{e.g.,\ }

\newcommand{\eps}{\varepsilon}

\newcommand{\R}{\mathbb{R}}

\newcommand{\N}{\mathbb{N}}

\DeclareMathOperator*{\argmax}{arg\,max}

\DeclareMathOperator{\bigoh}{\mathcal{O}}
\newcommand{\smalloh}{o}
\DeclareMathOperator{\card}{card}

\DeclareMathOperator{\ex}{\mathbb{E}}
\DeclareMathOperator{\expect}{\mathbb{E}}

\DeclareMathOperator{\one}{\mathds{1}}

\DeclareMathOperator{\prob}{\mathbb{P}}

\DeclareMathOperator{\supp}{supp}

\DeclarePairedDelimiter{\braces}{\{}{\}}
\DeclarePairedDelimiter{\bracks}{[}{]}
\DeclarePairedDelimiter{\parens}{(}{)}

\DeclarePairedDelimiterX{\braket}[2]{\langle}{\rangle}{#1,#2}

\DeclarePairedDelimiterX{\inner}[2]{\langle}{\rangle}{#1,#2}
\DeclarePairedDelimiterX{\setdef}[2]{\{}{\}}{#1:#2}

\DeclarePairedDelimiterXPP{\probof}[1]{\prob}{(}{)}{}{%

#1}

\DeclarePairedDelimiterXPP{\exof}[1]{\ex}{[}{]}{}{%

#1}

\usepackage[textwidth=30mm]{todonotes}

\newcommand{\debug}[1]{#1}

\theoremstyle{plain}
\newtheorem{theorem}{Theorem}

\newtheorem*{corollary*}{Corollary}
\newtheorem{lemma}[theorem]{Lemma}
\newtheorem{proposition}[theorem]{Proposition}

\theoremstyle{definition}
\newtheorem{definition}[theorem]{Definition}
\newtheorem*{definition*}{Definition}

\newtheorem*{hypothesis*}{Hypothesis}

\theoremstyle{remark}
\newtheorem{remark}[theorem]{Remark}
\newtheorem*{remark*}{Remark}
\newtheorem*{notation*}{Notational remark}

\numberwithin{theorem}{section}

\newcommand{\cdf}{\debug F}
\DeclareMathOperator{\E}{\mathbb{E}}
\DeclareMathOperator{\proba}{\mathbb{P}}

\newcommand{\rv}{\debug X}

\newcommand{\payv}{\debug v}
\newcommand{\payw}{\debug w}

\newcommand{\duration}{\debug x}
\newcommand{\timet}{\debug t}
\newcommand{\timeT}{\debug T}
\newcommand{\rew}{\debug r} 
\newcommand{\reward}{\debug U} 
\newcommand{\prof}{\debug c}
\newcommand{\profopt}{\debug {c^\ast}}
\newcommand{\profrew}{\debug \gamma}
\newcommand{\ubrv}{\debug C}
\newcommand{\lbrew}{\debug E}
\newcommand{\ubrew}{\debug D}
\newcommand{\pr}{\debug \lambda}
\newcommand{\exprv}{\debug S}
\newcommand{\totdur}{\debug \theta}
\newcommand{\reg}{\debug R}
\newcommand{\nume}{\debug \nu}
\newcommand{\den}{\debug \mu}
\newcommand{\quant}{\debug Q}

\newcommand{{\noisedrew}}{\debug Y}
\newcommand{\noise}{\debug \eps}
\newcommand{\bin}{\debug B}
\newcommand{\numbin}{\debug M}
\newcommand{\rewest}{\debug {\hat{\rew}}}
\newcommand{\opzero}{\debug \Phi}
\newcommand{\opzeroemp}{\debug {\hat{\opzero}}}
\newcommand{\opzeroest}{\debug {\bar{\opzero}}}
\newcommand{\profemp}{\debug {\hat{\prof}}}
\newcommand{\binsize}{\debug {h}}
\newcommand{\holdexp}{\debug \beta}
\newcommand{\holdcst}{\debug L}
\newcommand{\confrew}{\debug \eta}
\newcommand{\confprof}{\debug \xi}
\newcommand{\refuse}{\debug A}
\newcommand{\profrewest}{\debug {\hat{\profrew}}}
\newcommand{\countbin}{\debug N}
\newcommand{\thresh}{\debug s}
\newcommand{\threshopt}{\debug {\thresh^\ast}}
\newcommand{\profset}{\debug p}
\newcommand{\threshset}{\mathcal{\debug S}}
\newcommand{\confprofset}{\debug \zeta}
\newcommand{\numstage}{\debug S}

\title{Making the most of your day: online learning for optimal allocation of time}

\setcitestyle{authoryear}

\begin{document}
\maketitle

\begin{abstract}
We study online learning for optimal allocation when the resource to be allocated is time. 
An agent receives task proposals sequentially according to a Poisson process and can either accept or reject a proposed task. 
If she accepts the proposal, she is busy for the duration of the task and obtains a reward that depends on the task duration. 
If she rejects it, she remains on hold until a new task proposal arrives.
We study the regret incurred by the agent, first when she knows her reward function but does not know the distribution of the task duration, and then when she does not know her reward function, either. 
%
This natural setting bears similarities with contextual (one-armed) bandits, but with the crucial difference that the normalized reward associated to a context depends on the whole distribution of contexts.
\end{abstract}

\section{Introduction}\label{se:introduction}

\paragraph{Motivation.} 
A driver filling her shift with rides, a landlord renting an estate short-term, an independent deliveryman, a single server that can make  computations online, a communication system receiving a large number of calls, etc. all face the same trade-off. 
There is a unique resource that can be allocated to some tasks/clients for some duration. 
The main constraint is that, once it is allocated, the resource becomes unavailable for the whole duration. 
As a consequence, if a ``better'' request arrived during this time, it could not be accepted and would be lost.  
Allocating the resource for some duration has some cost but generates some rewards -- possibly both unknown and random. 
For instance, an estate must be cleaned up after each rental, thus generating some fixed costs; on the other hand, guests might break something, which explains why these costs are both random and unknown beforehand. Similarly, the relevance of a call is unknown beforehand. 
Concerning duration, the shorter the request the better (if the net reward is the same). 
Indeed, the resource could be allocated twice in the same amount of time. 

The ideal request would therefore be of short duration and large reward; this maximizes the revenue per time. 
A possible policy could be to wait for this kind of request, declining the other ones (too long and/or less profitable). 
On the other hand, such a request could be very rare. So it might be more rewarding in the long run to accept any request, at the risk of ``missing'' the ideal one.

Some clear trade-offs arise. The first one is between a greedy policy that accepts only the highest profitable requests -- at the risk of staying idle quite often -- and a safe policy that accepts every request -- but unfortunately also the non-profitable ones. 
The second trade-off concerns the learning phase; indeed, at first and because of the randomness, the actual net reward of a request is unknown and must be learned on the fly. 
The safe policy will gather a lot of information (possibly at a high cost) while the greedy one might lose some possible valuable information for the long run (in trying to optimize the short term revenue).

We adopt the perspective of an agent seeking to optimize her earned income for some large duration. 
The agent receives task proposals sequentially, following a Poisson process. When a task is proposed, the agent observes its expected duration and can then either accept or reject it.  
If she accepts it,  she cannot receive any new proposals for the whole duration of the task.  
At the end of the task she observes her reward, which is a function of the duration of the task.
If, on the contrary, she rejects the task, she remains on hold until she receives a new task proposal.

The agent's policies are evaluated in terms of their expected regret, which is the difference  between the cumulative rewards  obtained until $\timeT$ under the optimal policy and under the implemented agent policy (as usual the total length could also be random or unknown \citep{anytime}). 
In this setting, the ``optimal'' policy is within the class of policies that accept -- or not -- tasks whose length belongs to some given acceptance set (say, larger than some threshold, or in some  specific Borel subset, depending on the regularity of the reward function).

\paragraph{Organization and main contributions.}

In \cref{se:model}, we formally introduce  the model and the problem faced by an oracle who knows the distribution of task durations as well as the reward function (quite importantly, we emphasize again that the oracle policy must be independent of the realized rewards). 
Using  continuous-time dynamic programming principles, we construct a quasi-optimal policy in terms of accepted and rejected tasks: this translates into a single optimal threshold for the ratio of the reward to the duration, called the profitability function. 
Any task with a profitability above this threshold is accepted, and the other ones are declined.
As a benchmark, we first assume in \cref{subsec:unknown-distrib} that  the agent knows the reward function $\rew(\cdot)$, but ignores the distribution of task durations. The introduced techniques can be generalized, in the following sections, to further incorporate estimations of $\rew(\cdot)$. 
In that case, our base algorithm has a regret  scaling as $\bigoh(\sqrt{\timeT})$; obviously, this cannot be improved without additional assumptions, ensuring minimax optimality. 
In \cref{subsec:bandit}, the reward function is not known to the agent anymore and the reward realizations are assumed to be noisy. 
To get  non-trivial estimation rates, regularity -- \ie  $(\holdcst,\holdexp)$-H\"older -- of the reward function is assumed. 
Modifying the basic algorithm to incorporate non-parametric estimation of $\rew$ yields  a regret scaling as $\bigoh(\timeT^{1-\eta}\sqrt{\ln\timeT})$ where $\eta =\holdexp/(2\holdexp+1)$. As this is the standard error rate in classification \citep{tsybakov2006statistique}, minimax optimality (up to  $\log$-term) is achieved again.  
Finally, our different algorithms are empirically evaluated on simple toy examples in \cref{sec:expe}. Due to space constraints, all the proofs are deferred to the Appendix.

\paragraph{Related work.} 

As data are gathered sequentially, our problem bears similarities with online learning and multi-armed bandit \citep{bubeck2012regret}. 
The main difference with multi-armed bandit or resource allocation problems \citep{fontaine} is that the agent's only resource is her time, which has to be spent wisely and, most importantly, saving it actually has some unknown value. 
The problem of time allocation actually goes way back.
In economic theory, Becker's seminal paper \citep{becker1965atheory} evaluates the full costs of non-working activities as the sum of their market prices and the forgone value of the time used to enjoy the activities. 
Various empirical works have followed \citep[see, \eg][]{thomas1991theallocation,chabris2009theallocation}.

The problem of online scheduling has been studied in several possible variations \citep[see, \eg][for a survey]{pruhs2004online}. 
Various subsequent papers consider scheduling models that combines online and stochastic aspects \citep[see, \eg][]{MegUetVre:MOR2006,ChoLiuQueSim:OR2006,Vredeveld:2012tw,MarRutVre:AOA2012,SkuSviUet:MOR2016}. Yet these models are different as they aim at minimizing the makespan of all received jobs, while actions here consist in accepting/declining tasks.
%
%
In online admission control, a controller receives requests and decides on the fly to accept/reject them. However, the problem of admission control is more intricate since a request is served through some path, chosen by the controller, in a graph. Even solving it offline thus remains a great challenge \citep{wu2001admission, leguay2016admission}. 

An online model of reward maximization over time budget has been proposed by \citet{cayci2019learning,cayci2020budget}, where the agent does not observe the duration of a task before accepting it. 
As a consequence, the optimal policy always chooses the same action, which is maximizing the profitability. 
Since the duration of a task is observed before taking the decision in our problem, the decision of the optimal policy depends on this duration, used as a covariate. 
\citet{cayci2019learning} thus aim at determining the arm with the largest profitability, while our main goal is here to estimate the profitability threshold from which we start accepting the tasks. The considered problems and proposed solutions thus largely differ in these two settings.

Our problem is strongly related to contextual multi-armed bandits, where each arm produces a noisy reward that depends on an observable context. Indeed,  the agent faces a one arm contextual bandit problem \citep{Sar:AS1991}, with the crucial difference that the normalized reward associated to a context depends on the whole distribution of contexts (and not just the current context).
The literature on contextual bandits, also known as bandits with covariates, actually goes back to \citet{Woo:JASA1979}
and has seen extensive contributions in different settings \citep[see, \eg][]{YanZhu:AS2002,WanKulPoo:IEEETAC2005,rigollet2010nonparametric,
GolZee:AAP2009,PerRig:AS2013}.

This problem is also related to bandits with knapsacks \citep[see][Chapter 10]{slivkins2019} introduced by~\citet{badanidiyuru2013} and even more specifically to contextual bandits with knapsacks \citep{badanidiyuru2014, agrawal2016}, where pulling an arm consumes a limited resource. Time is the resource of the agent here, while  both time and resource are well separated quantities in bandits with knapsacks. Especially, bandits with knapsacks assume the existence of a null arm, which does not consume any resource. This ensures the feasibility of the linear program giving the optimal fixed strategy. Here, the null arm (declining the task) still consumes the waiting time before receiving the next task proposal. Bandits with knapsacks strategies are thus not adapted to our problem, which remains solvable thanks to a particular problem structure.
The problem of online knapsacks is also worth mentioning \citep{noga2005online,chakrabarty2008online,han2009online,bockenhauer2014online}, where the resources consumed by each action are observed beforehand. 
This is less related to our work, as online knapsacks consider a competitive analysis, whereas we here aim at minimizing the regret with stochastic contexts/rewards.

\section{Model and benchmark}\label{se:model}

\subsection{The problem: allocation of time}\label{subsec:driver-problem}

All the different notations used in the paper are summarized at the beginning of the Appendix to help the reader.
We consider an agent who sequentially receives task proposals and decides whether to accept or decline them on the fly. 
The durations of the proposed tasks are assumed to be i.i.d. with an unknown law, and  $\rv_i$ denotes the duration of the $i$-th task.  
If this task is accepted, the agent earns some  reward $Y_i$ 
 with expectation $\rew(\rv_i)$. 
The profitability is then defined as the function $x \mapsto \rew(x)/x$.
We emphasize here that $Y_i$ is not observed before accepting (or actually completing) the $i$-th task and that  the expected reward function $\rew(\cdot)$ is  unknown to the agent at first. 
If task $i$ is accepted, the agent cannot accept any new proposals for the whole duration $\rv_i$.
We assume in the following that the function $\rew$ is bounded in $[\lbrew,\ubrew]$ with $\lbrew \leq 0 \leq \ubrew$, and that the durations $\rv_i$ are upper bounded by $\ubrv$.

After completing the $i$-th task, or after declining it, the agent is on hold, waiting for a new proposal. 
We assume that idling times  -- denoted by $\exprv_i$ -- are also i.i.d.,  following some exponential law of parameter $\pr$. 
This parameter is supposed to be known beforehand as it has a small impact on the learning cost (it can be quickly estimated, independently of the agent's policy). 
An  equivalent formulation of the task arrival process is that proposals follow a Poisson process (with intensity $\pr$) and the agent does not observe task proposals while occupied. 
This is a mere consequence of the memoryless property of Poisson processes.

The agent's objective is to maximize the expected sum of rewards obtained by choosing an appropriate acceptance policy. 
Given the decisions $(a_i)_{i\geq 1}$ in $\{0,1\}$ (decline/accept), the total reward accumulated by the agent after the first $n$ proposals is equal to 
$\sum_{i=1}^n Y_i a_i$  and the required amount of time for this is $\mathcal{T}_n \coloneqq \sum_{i=1}^n \exprv_i + \rv_i a_i$. 
This amount of time is random and strongly depends on the policy. As a consequence, we consider that the agent optimizes the cumulative reward up to time $\timeT$, so that the number of received tasks proposals is random and equal to 
$\totdur \coloneqq \min\{n \in \N \mid \mathcal{T}_n> \timeT\}.$
Mathematically, a policy of the agent is a function $\pi$ -- taking as input a proposed task $X$, the time $t$ and the history of observations $\mathcal{H}_t \coloneqq (\rv_{i}, Y_i a_i)_{i: \mathcal{T}_i < t}$ -- returning the decision $a \in \left\{ 0, 1 \right\}$ (decline/accept). In the following, the optimal policy refers to the strategy $\pi$ maximizing the expected reward $\reward_\pi$ at time $T$, given by
\begin{equation} 
\label{eq:defnreward}
\reward_\pi(T) = \expect\left[ \sum_{n=1}^\theta r(X_n) \pi(X_n, t_n, \mathcal{H}_{t_n})\right],
\end{equation}%
where $t_n \coloneqq \exprv_n + \sum_{i=1}^{n-1} \exprv_i + \rv_i a_i$ is the time at which the $n$-th task is proposed.

We allow a slight boundary effect, \ie the ending time can slightly exceed $T$, because the last task may not end exactly at time $\timeT$. 
A first alternative would be to compute the cumulative revenue over completed tasks only, and another  alternative would be to attribute the relative amount of time before~$T$ spent on the last task. 
Either way, the boundary effect is of the order of a constant and has a negligible impact as $\timeT$ increases.

\subsection{The benchmark: description of the optimal policy}\label{subse:benchmark}

The benchmark to which the real agent is compared is an oracle who knows beforehand both the distribution of tasks and the reward function. It is easier to describe the reward of the optimal policy than the policy itself. 
Indeed, at the very last time $\timeT$, the agent cannot accumulate any more reward, hence the remaining value is 0. We then compute the reward of the policy backward, using continuous dynamic programming principles.   
Formally, let  $\payv(\timet)$ denote the ``value'' function, \ie the expected reward that the optimal policy will accumulate on the remaining time interval $[\timet,\timeT]$ if the agent is on hold at this very specific time $\timet \in [0,\timeT]$. 
As mentioned before, we assume the boundary condition $\payv(\timeT)=0$. 
The next proposition gives  the dynamic programming equation satisfied by $\payv(\timet)$ (with the notation  $(z)_+ = \max\{z, 0\}$ for any $z \in \R$), as well as a simple function approximating this value.

\begin{proposition}\label{prop:dynamic-programming}
The value function $\payv$ satisfies the dynamic programming equation
\begin{equation}
\label{eq:dynprog}	
\begin{cases}
        \payv'(\timet) = -\pr \E\left[\left(\rew(\rv) + \payv(\timet + \rv) - \payv(\timet)\right)_+\right] \, &\text{ for all } \timet < \timeT,\\
       \payv(t) = 0 \, &\text{ for all } \timet \geq \timeT.
\end{cases}
\end{equation}

If $\rv\le\ubrv$, then its  solution is bounded by the affine function $\payw : \timet \mapsto \profopt(\timeT-\timet)$ for any $t \in [0, T]$ as follows
\begin{equation}
\label{eq:w}
w(t-\ubrv) \geq v(t) \geq w(t),
\end{equation}
where $\profopt$ is the unique root of the function 
\begin{equation}
\label{eq:opzero}
\opzero : 
\begin{array}{l}
\R_+ \to \R\\
\prof \mapsto \pr \E\left[\left(\rew(\rv) - \prof \rv \right)_+\right] - \prof.
\end{array}
\end{equation}
\end{proposition}

The constant $\profopt$ represents the optimal reward per time unit, hereafter referred to as the \textit{optimal profitability threshold}. The function $\payw$ is the value function of the optimal policy when neglecting boundary effects.
The proof of \cref{prop:dynamic-programming} largely relies on the memoryless property of the idling times and determining a benchmark policy without this memoryless assumption becomes much more intricate.
Based on this, it is now possible to approach the optimal policy by only accepting task proposals with profitability at least $\profopt$. This results in a stationary policy, \ie its decisions depend neither on the time nor on past observations but only on the duration of the received task.

\begin{theorem}\label{thm:baseline}
The policy $\pi^*$ which accepts a task  with duration $\duration$ if and only if $\rew(\duration)\geq \profopt \duration$ is $\profopt \ubrv$-suboptimal, \ie $\reward_{\pi^*} \geq \max_\pi \reward_{\pi} - \profopt \ubrv$.
\end{theorem}

In the considered model, it is thus possible to compute a quasi-optimal online policy. 

\subsection{Online learning policies and regret}\label{se:learning}

Below we investigate  several contexts where the agent lacks information  about the environment, such  as the task durations distribution and/or the (expected) reward function. As mentioned before, the objective of the agent is to maximize the cumulative reward gathered until~$\timeT$ (up to the same boundary effect as the optimal policy) or equivalently to minimize the policy regret, defined for the policy $\pi$ as
\begin{equation}
\label{eq:regret}
\reg(\timeT) = \profopt\timeT - \reward_\pi(T).
\end{equation}
Note that $\reg(\timeT)$ is the difference between the expected rewards of strategies $\pi^*$ and $\pi$, up to some constant term due to the boundary effect.
\section{Warm up: known reward, unknown distribution}\label{subsec:unknown-distrib}

First, we assume that the reward function $\rew$ is known to the agent, or equivalently, is observed along with incoming task proposals. 
However, the agent does not know the distribution $\cdf$ of task durations  and we emphasize here that the agent does not observe incoming task proposals while occupied with a task (though the converse assumption should only affect the results by a factor~$\ubrv^{-1}$, the inverse of the maximal length of a task). We now define
\begin{equation}
\label{eq:opzero_n}
\opzero_n : \prof\mapsto \pr\frac{1}{n}\sum_{i=1}^n \left(\rew(\rv_i)-\prof\rv_i\right)_+ - \prof,
\end{equation}
which is the empirical counterpart of $\opzero$.
Moreover, let $\prof_n$ be the unique root of $\opzero_n$.
One has $\E[\opzero_n(\prof)] = \opzero(\prof)$ for all $\prof\geq 0$ and all $n\geq 1$.

\begin{proposition}
\label{prop:prof_estimate}
For all $\delta\in(0,1]$ and $n\geq 1$, $\prob\Big(\prof_n-\profopt> \pr (\ubrew-\lbrew) \sqrt{\frac{\ln(1/\delta)}{2n}} \Big) \leq \delta$ .
\end{proposition}

\begin{remark}
\cref{prop:prof_estimate} states that the error in the estimation of $\profopt$ scales with $\pr (\ubrew-\lbrew)$. Notice that if the reward function is multiplied by some factor (and $\pr$ is fixed), then $\profopt$ is multiplied by the same factor; as a consequence, a linear dependency in $\ubrew-\lbrew$ is expected. \\
Similarly, since $\profopt=\pr \E\left[\left(\rew(\rv) - \profopt \rv \right)_+\right]$, a small variation of $\pr$ induces a variation of $\profopt$ of the same order. As a consequence, a linear dependency in $\pr$ (for small values of $\pr$) is expected. And as both effects are multiplied, the dependency in $\pr (\ubrew-\lbrew)$ is correct. \\
On the other hand, as $\pr$ goes to infinity, $\profopt$ converges to $\max_x r(x)/x$, so our deviation result seems irrelevant (for a small number of samples). 
This is due to the fact that for large $\pr$, $\profopt$ is not really the expectation of a random variable, but its essential supremum. In the proof, this appears when   $\opzero(\profopt+\eps)$ is bounded by $- \eps$. It is not difficult to see that a tighter upper-bound is 
\begin{equation}
\label{eq:q-epsilon}
-\eps(1+\pr q_\eps), \quad\text{with}\quad
q_\eps \coloneqq
\E \bracks*{X\one(\rew(\rv)\geq (\profopt + \eps)\rv)}.
\end{equation}
Hence the approximation error  scales with $\pr (\ubrew-\lbrew)/(1+\pr q_\eps)$ as $\pr$ goes to infinity. 
However, this does not give explicit confidence intervals and the regime $\lambda \to \infty$ is  not really interesting. As a consequence, the formulation of \cref{prop:prof_estimate} is sufficient for our purposes.
\end{remark}

The used policy  is straightforward: upon receiving the $n$-th proposal, the agent computes $\prof_n$ and accepts the task $\rv_n$ if and only if $\rew(\rv_n)\geq \prof_n\rv_n$. The pseudo-code of this policy is given in \cref{alg:unknown-distribution} below. 
The regret of this policy can be rewritten as
\begin{gather}
\label{eq:regret-} 
\reg(\timeT) 
= \profopt\timeT - \E\left[\sum_{n=1}^{\totdur} \profrew_n\E\left[\rv_n\one(\rew(\rv_n)\geq \prof_n\rv_n) + \exprv_n\right]\right], \\
\label{eq:gamma-n}
\text{where }\quad\profrew_n \coloneqq \frac{\pr \E[\rew(\rv_n)\one(\rew(\rv_n)\geq\prof_n\rv_n)]}{1+\pr\E[\rv_n\one(\rew(\rv_n)\geq \prof_n\rv_n)]} = \frac{ \E[\rew(\rv_n)\one(\rew(\rv_n)\geq\prof_n\rv_n)]}{\E[S_n + \rv_n\one(\rew(\rv_n)\geq \prof_n\rv_n)]}
\end{gather}
corresponds to the expected per-time reward obtained for the $n$-th task under the policy induced by $(\prof_n)_{n\geq 1}$. In the last expression, the numerator is indeed the expected reward per task while the denominator is the expected time spent per task (including the waiting time before receiving the proposal).
 \cref{prop:proprew}, postponed to the Appendix, gives a control of $\gamma_n$ entailing the following.


\begin{theorem}\label{th:regret-fullinfo}
In the known reward, unknown distribution setting, the regret of \cref{alg:unknown-distribution} satisfies
\begin{align}
\reg(\timeT) \leq  \pr(\ubrew-\lbrew)\ubrv\sqrt{\frac{\pi}{2}}\sqrt{\pr\timeT+1}.
\end{align}
\end{theorem}

\begin{remark}
Once again, one might question the dependency on the different parameters. As mentioned before, $\pr (\ubrew-\lbrew)$ represents the scale at which $\profopt$ grows and $\ubrv$ is the scale of $\rv$, so that $\pr(\ubrew-\lbrew)\ubrv$ is the total global scaling of rewards. On the other hand, the parameter $\pr$ also appears in the square-root term. This is because $\pr \timeT$ is approximately the order of magnitudes of the observed tasks.
%
\end{remark}

\begin{algorithm}[h]
\DontPrintSemicolon
\SetKwInOut{Input}{input}
\Input{ $\rew$, $\pr$}

	$n=0$\;
	\While{$\sum_{i=1}^n \exprv_i + \rv_i\one(\rew(\rv_i)\geq \prof_i\rv_i)<\timeT$}{
		$n=n+1$\;
		Wait $\exprv_n$ and observe $\rv_n$\;
		Compute $\prof_n$ as the unique root of $\opzero_n$ defined in \cref{eq:opzero_n}\;
		\lIf{$\rew(\rv_n)\geq \prof_n \rv_n$}{
				accept task and receive reward $\rew(\rv_n)$}
				\lElse{reject task}		
	}
\caption{\label{alg:unknown-distribution}Known reward algorithm}
\end{algorithm}

\begin{remark}\label{rem:computation1}
The estimated profitability threshold $\prof_n$ can be efficiently approximated using binary search, since the function $\opzero_n$ is decreasing. The approximation error is ignored in the regret, as approximating $\prof_n$ up to, \eg $n^{-2}$ only leads to an additional constant term in the regret. \\
The computation and memory complexities of \cref{alg:unknown-distribution} are then respectively of order $n\log(n)$ and $n$ when receiving the $n$-th proposal, because the complete history of tasks is used to compute $c_n$, the root of $\Phi_n$. 
This can be improved by noting that only the tasks with profitability in 
\begin{equation*}
\bracks*{c_n - \pr (\ubrew-\lbrew) \sqrt{\frac{\ln(1/\delta)}{2n}}, \ c_n + \pr (\ubrew-\lbrew) \sqrt{\frac{\ln(1/\delta)}{2n}}}
\end{equation*}
need to be exactly stored, thanks to \cref{prop:prof_estimate}. 
Indeed, with high probability, tasks with smaller profitability do not contribute to $\Phi(\profopt)$, while tasks with larger profitability fully contribute to $\Phi(\profopt)$, meaning that only their sum has to be stored. 
As this interval is of length $1/\sqrt{n}$, the complexities of the algorithm become sublinear in $n$ under regularity assumptions on $X$ and the profitability function. 
The computational complexity can even be further improved to $\log(n)$ using binary search trees, as explained in \cref{sec:fastcomp}.
\end{remark}

\section{Bandit Feedback}\label{subsec:bandit}

We now assume that rewards are noisy and not known upfront. 
When the agent accepts a task~$\rv_i$, the reward  earned and observed over the period of time $\rv_{i}$ is $\noisedrew_i=\rew(\rv_i)+\noise_i$, where $(\noise_i)_{i\geq1}$ is a sequence of $\sigma^2$-subgaussian independent random variables with $\E[\noise_i]=0$. 
If the task  $\rv_{i}$ is rejected, then the associated reward is not observed; on  the other hand, the agent is not occupied and thus may be proposed another task, possibly within the time window of length $\rv_{i}$.

To get non-trivial estimation rates,  in the whole section we assume the reward function $\rew$ to be $(\holdexp,\holdcst)$-H\"older, whose definition is recalled below.
\begin{definition}
\label{def:holder}
Let $\holdexp\in(0,1]$ and $\holdcst>0$. The function $\rew$ is  $(\holdexp,\holdcst)$-H\"older if
\begin{align*}
\vert\rew(\duration)-\rew(\duration')\vert \leq \holdcst\vert\duration-\duration'\vert^\holdexp, \qquad \forall \duration,\duration'\in [0,\ubrv]\ .
\end{align*}
\end{definition}

The reward function $\rew(\cdot)$ is estimated by $\rewest_n(\cdot)$, constructed as a variant of a regressogram. 
The state space   $[0,\ubrv]$ is partitioned regularly into $\numbin$ bins $\bin_1,\dots,\bin_\numbin$ of size $\binsize = \ubrv/\numbin$. 
For every bin~$\bin$, let $x^\bin \coloneqq \min \left\{ x \in B \right\}$.
Similarly to \cref{alg:unknown-distribution}, $\profopt$ is estimated by $\profemp_n$. 
To define the learning algorithm, we also construct an upper-estimate of $\rew(\cdot)$, denoted by $\rewest_n^+(\cdot)$, and a lower-estimate of $\profopt$, denoted by  $\profemp_n^-$. 

The learning algorithm is actually quite simple.  
A task $\rv_n$ is accepted if and only if its best-case reward $\rewest_{n-1}^+(\rv_n)$ is bigger than the worst-case per-time value of its bin $\profemp_{n-1}^-x^{\bin(\rv_n)}$, where $\rv_n$ is in the bin $\bin(\rv_n)$. 
Notice that, if $\rewest_n^+(\cdot)$ is  bigger than $\rew(\cdot)$ and  $\profemp_n^-$ is always smaller than $\profopt$, then any task accepted by the optimal algorithm (\ie such that $\rew(\rv) \geq \profopt\rv$) is accepted by the learning algorithm. Hence regret is incurred solely by accepting sub-optimal tasks. 
The question is therefore how fast tasks can be detected as sub-optimal, and whether declining them affects, or not, the estimation of $\profopt$. The pseudo-code is given in \cref{alg:bandit} at the end of this section.

We now describe the construction of $\rewest_n$ and $\profemp_n$ (as well as their upper and lower-estimates). 
Let   $(\rv_1,\noisedrew_1,\dots,\rv_n,\noisedrew_n)$  be the vector of   observations  of pairs (task, reward). 
We define for all bins $\bin$, 
\begin{equation}
\label{eq:r-n-B}
\begin{gathered}
\rewest_n(\duration)=\rewest_n^{\bin} = \frac{1}{\countbin_\bin}\sum_{i=1}^n \noisedrew_i\one(\rv_i\in \bin \text{ and } a_i=1), \quad \forall \duration \in \bin \\
\text{with } \countbin_\bin ={\textstyle\sum_{i=1}^n} \one(\rv_i\in \bin \text{ and } a_i=1).
\end{gathered}
\end{equation}
We emphasize here that declined tasks are not used in the estimation of $\rew(\cdot)$ (as we shall see later, they are not used  in the estimation of $\profopt$, either). 
The upper estimate of $\rew(\cdot)$ is now defined as
\begin{equation}\label{eq:eta-n-1}
\rewest^+_n(\duration) = \rewest_n(\duration) + \underbrace{\sqrt{\sigma^2 + \frac{L^2}{4}\left(\frac{\ubrv}{\numbin}\right)^{2\beta}}\sqrt{\frac{\ln(\numbin/\delta)}{2\countbin_\bin}}+ L\left(\frac{\ubrv}{\numbin}\right)^\beta}_{\coloneqq \confrew_{n-1}(x)}, \quad \forall x \in \bin.
\end{equation}

The following \cref{lem:rewest} states that $\rewest^+_n(\cdot)$ is indeed an upper-estimate of $\rew(\cdot)$.
\begin{lemma}\label{lem:rewest}
For every $n \in \N$, we have
$\prob\left(\forall x \in [0,\ubrv],\ \rewest^+_n(\duration) \geq \rew(\duration)\right)\geq 1- \delta.$
\end{lemma}

It remains to construct $\profemp_n$. 
We  define iteratively for every bin $B$
\begin{equation}
\label{eq:r-tilde}
\tilde{\rew}_n^{\bin} =
\begin{cases} 0 & \text{ if a task in } \bin \text{ has ever been rejected},  \\ 
\rewest_n^{\bin} & \text{ otherwise.}
\end{cases}
\end{equation}
So $\tilde{\rew}_n$ is equal to the reward estimate $\rewest_n$, except on the eliminated bins that are known to be suboptimal, for which it is instead $0$.
We then introduce  the empirical counterpart of $\opzero$,
\begin{equation}\label{eq:opzeroemp_n}
\opzeroemp_n : \prof\mapsto \pr\sum_{j=1}^{\numbin} \frac{\countbin_{\bin_j}}{n} \left(\tilde{\rew}_n^{\bin_j}-\prof x^{\bin_j}\right)_+ - \prof.
\end{equation}
Let $\profemp_n$ be the unique root of $\opzeroemp_n$; the lower estimate of $\profopt$ is then 
\begin{equation}
\label{eq:xi-n-1} 
\profemp^-_n = \profemp_n - \underbrace{\parens*{2\pr\sqrt{\sigma^2 + \frac{(\ubrew-\lbrew)^2}{4}}\sqrt{\frac{\ln(1/\delta)}{n}} + \kappa\pr\max\left(\sigma, \frac{\ubrew-\lbrew}{2}\right)\sqrt{\frac{\log(n)+1}{hn}} +\sqrt{8}\pr\frac{\holdcst\binsize^{\holdexp}}{2^\holdexp}+ \pr^2 \ubrew h}}_{\coloneqq \confprof_{n-1}},
\end{equation}
where $\kappa\leq 150$, is the square root of the universal constant introduced in \citep[][Theorem 11.3]{GyoKohKrzWal:Springer2002}. 
Although the following lemma looks trivial, it is actually a keystone of our main result. 
It states that the optimal profitability threshold may be computed solely based on the accepted tasks  -- no matter what the rewards associated to declined tasks are. 
This is of crucial relevance, since it implies that a learning algorithm may actually decline tasks that are clearly non-profitable, without degrading the quality and the speed of profitability threshold estimation.

\begin{lemma}
\label{lem:decline}
Let $\profopt$ be the optimal profitability threshold associated to $\rew$, \ie $\profopt$ is the unique root of $\prof\mapsto \opzero(\prof)=\pr\E\left[\left(\rew(\rv) - \prof \rv \right)_+\right] - \prof$ and let $\mathcal{E} \subset \{x \in [0,\ubrv] \,\vert\, \rew(x) \leq \profopt x \}$ be any subset of sub-profitable tasks.\\
Then the profitability threshold $\tilde{\prof}$ associated to the modified reward function $\rew_{\mathcal{E}}(x)\coloneqq\rew(x)\one(x \not\in \mathcal{E})$ is equal to $\profopt$, or, stated otherwise,
\begin{equation*}
\profopt = \pr\E\left[\left(\rew_\mathcal{E}(\rv)- \profopt \rv \right)_+\right], \quad \forall \mathcal{E} \subset  \{x \in [0,\ubrv] \,\vert\, \rew(x) \leq \profopt x \}.
\end{equation*}
\end{lemma}
This implies that $\profemp_n^-$ is indeed an under-estimate. \begin{proposition}\label{prop:giveaname0}
For all $N \in \N$, we have  $\prob(\forall n \leq N, \profemp_n^- \leq \profopt   ) \geq 1- 2N\delta$. 
\end{proposition}
%
%
We can now state our main result.

\begin{theorem}\label{th:regret-bandit}
If $\rew$ is $(L,\beta)$-H\"older, then the regret of \cref{alg:bandit},  taking $\delta = 1/\timeT^2$ and $\numbin = \big\lceil \ubrv \holdcst^{\frac{2}{2\holdexp+1}}(\pr \timeT +1)^{\frac{1}{2\holdexp+1}} \big\rceil$, satisfies
\begin{align*}
\reg(\timeT) \leq \kappa_1 \pr\ubrv\max(\sigma, \ubrew-\lbrew)\holdcst^{\frac{1}{2\beta+1}}
(\pr\timeT+1)^\frac{\holdexp+1}{2\holdexp+1}\sqrt{\ln(\pr\timeT+1)},
\end{align*}
where $\kappa_1$ is some universal constant independent of any problem parameter. 
\end{theorem}

A similar lower bound is proved in \cref{subsec:lower}, which implies that \cref{alg:bandit} optimally scales with $T$, up to $\log$ terms. Even faster rates of convergence hold under additional assumptions on the task distribution or the profitability function and are provided in \cref{sec:fast} for the cases of finite support of $X$, margin condition or monotone profitability function.

\begin{algorithm}[h]
\DontPrintSemicolon
\SetKwInOut{Input}{input}
\Input{$\delta$, $\pr$, $\ubrv$, $\ubrew$, $\lbrew$, $\timeT$, $\sigma$, $\kappa$}
	$\profemp^-_0=0$;\ $\rewest^+_{0} =\infty$;\ $n=0$\;
	\While{$\sum_{i=1}^n \exprv_i + \rv_i\one(\rewest^+_i(\rv_i)\geq \profemp^-_i\rv_i)<\timeT$}{
		$n=n+1$\;	
		Wait $\exprv_n$ and observe $\rv_n$\;
		 		\lIf{$\rewest^+_{n-1}(\rv_n)\geq \profemp^-_{n-1}x^{\bin(\rv_n)}$}{
					accept task and receive reward $Y_n$}
					\lElse{reject task}
				For $\bin=\bin(\rv_n)$ compute $\rewest_n^\bin, \rewest^{\bin\, +}_n, \tilde{\rew}_n^\bin$ as described by \cref{eq:r-n-B,eq:eta-n-1,eq:r-tilde} \;
Compute $\profemp_n$ as the unique root of $\opzeroemp_n$ defined in \cref{eq:opzeroemp_n}\;
 	$\profemp^-_n = \profemp_n - \confprof_{n-1}$ as described by \cref{eq:xi-n-1}
	}
\caption{\label{alg:bandit}Bandit algorithm}
\end{algorithm}

\begin{remark}
The complexity of \cref{alg:bandit} when receiving a single task scales with the number of bins $\numbin$, because the function $\opzeroemp_n$ uses the discretization over the bins. 
In particular, it uses the representatives $x^\bin$ instead of the exact observations $\rv_i$ in $\opzeroemp_n$, which is the reason for the additional $\pr^2 \ubrew h$ term in $\confprof_n$.
Similarly to \cref{alg:unknown-distribution}, it can be improved to a $\log(\numbin)$ computational complexity with binary search trees as detailed in \cref{sec:fastcomp}.
\end{remark}

\section{Simulations}\label{sec:expe}

\cref{alg:unknown-distribution,alg:bandit} are computationally efficient when using binary search trees as described in \cref{sec:fastcomp}.
This section compares empirically these different algorithms on toy examples, considering affine and concave reward functions.
The code used in this section is available at \url{github.com/eboursier/making_most_of_your_time}.

\subsection{Affine reward function}\label{sec:expeaffine}

We first study the simple case of affine reward function $\rew(x) = x - 0.5$. The profitability function is increasing and \cref{alg:monotone}, given in \cref{SE:FastMonotone}, can be used in this specific case.
%
We consider a uniform distribution on $[0,3]$ for task distribution, a Poisson rate $\pr = 1$ and a uniform distribution on $[-1,1]$ for the noise $\varepsilon$.
As often in bandit algorithms, we tune the constants scaling the uncertainty terms $\confrew$ and $\confprof$ for better empirical results.

\medskip

\cref{fig:decisions1} shows the typical decisions of \cref{alg:unknown-distribution,alg:bandit} on a single run of this example. 
The former especially seems to take almost no suboptimal decision. This is because in simpler settings with good margin conditions, the regret of \cref{alg:unknown-distribution} can be shown to be only of order $\log(T)$. 
On the other hand, it can be seen how \cref{alg:bandit} progressively eliminates bins, represented by red traces.

\cref{fig:bandit1} gives the upper estimate of the reward function by representing $\rewest_n^{\bin \, +}/x^\bin$ for each bin and the lower estimate of $\profemp_n^-$ after a time $t=10^5$. It illustrates how the algorithm takes decisions at this stage. Especially, the bins with estimated profitability above $\profemp_n^-$, while their actual profitability is below $\profopt$, still have to be eliminated by the algorithm.
Some bins are badly estimated, since the algorithm stops accepting and estimating them as soon as they are detected suboptimal. The algorithm thus adapts nicely to the problem difficulty, by eliminating very bad bins way faster than slightly suboptimal ones.

\begin{figure}[h]
\begin{subfigure}{.45\linewidth}
\centering
\resizebox{\linewidth}{!}{\includegraphics[scale=1]{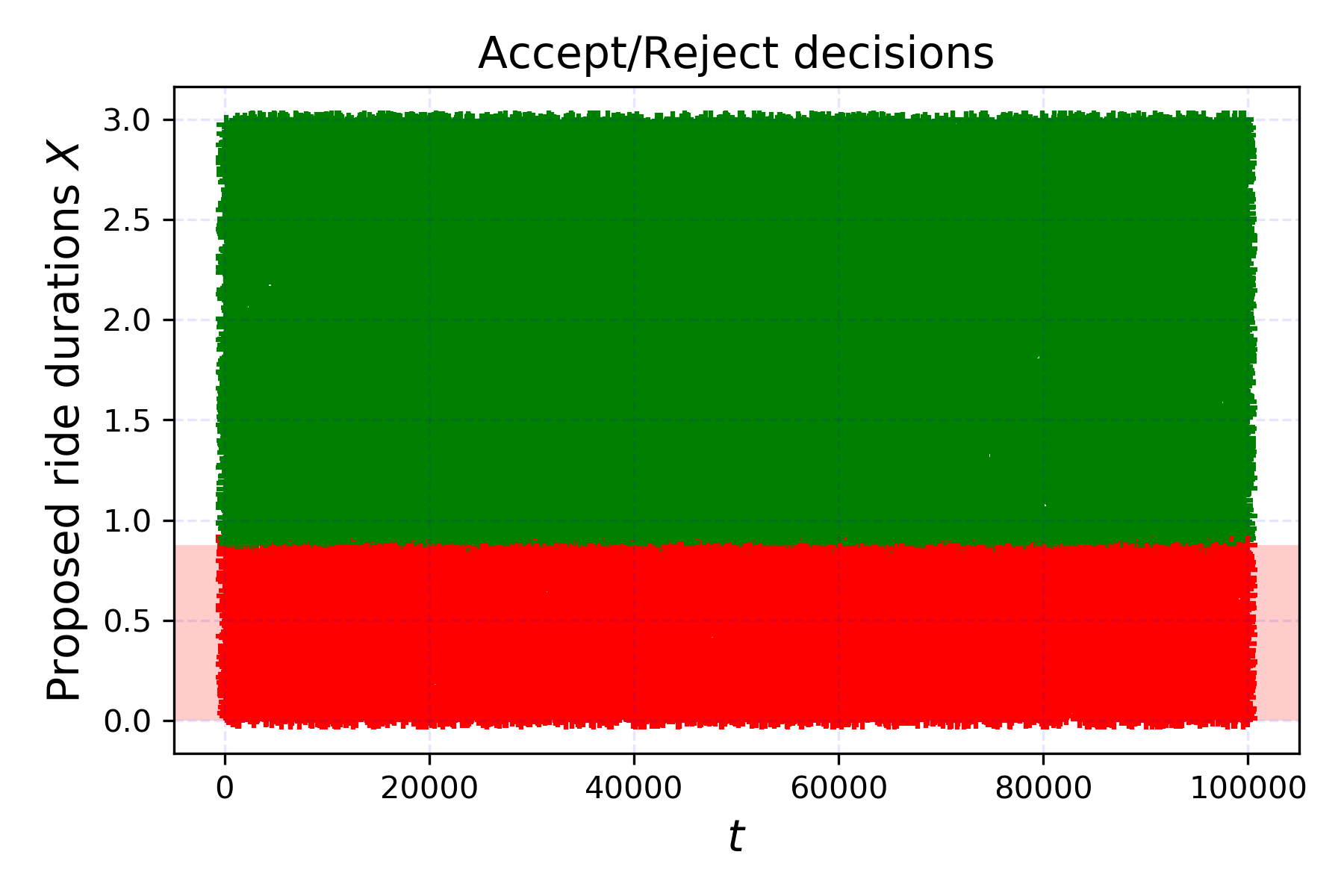}}
\caption{\cref{alg:unknown-distribution}: known reward.}
\end{subfigure}%
\hspace{0.2cm}
\begin{subfigure}{.45\linewidth}
\centering
\resizebox{\linewidth}{!}{\includegraphics[scale=1]{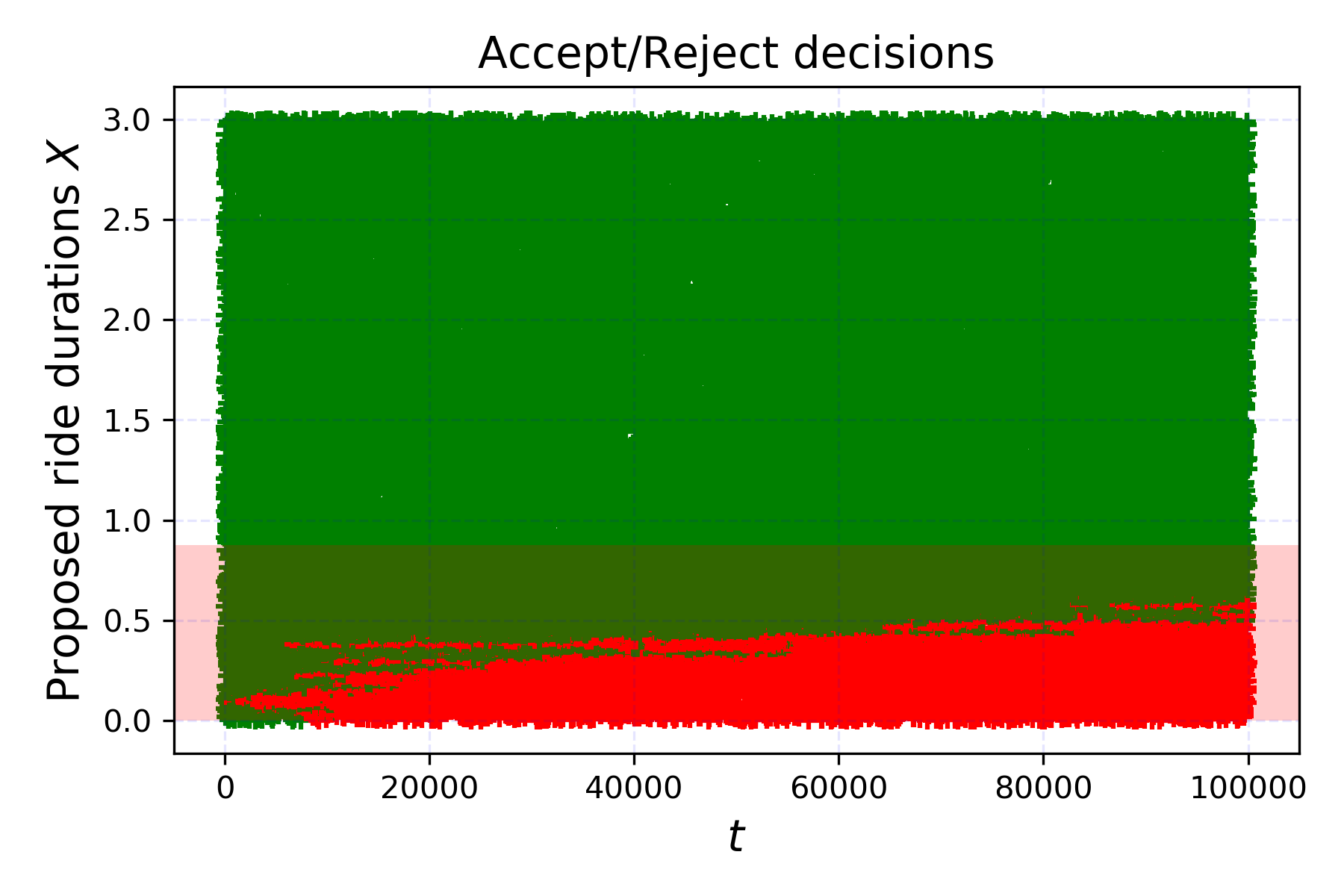}}
\caption{\cref{alg:bandit}: bandit.}
\end{subfigure}
\caption{\label{fig:decisions1}Accept/reject decisions on a single run. A single point corresponds to a proposed task, with the time of appearance on $x$ axis and task durations on $y$ axis. It is marked in green (resp. red) if it is accepted (resp. rejected) by the algorithm, while the light red area highlights the suboptimal region. Each task in this colored region thus has a profitability below $\profopt$.}
\end{figure}

\begin{figure}[h]
\begin{adjustwidth}{-0.2cm}{-0.2cm}  
\centering
\begin{minipage}{.47\linewidth}
\begin{figure}[H]
  \includegraphics[width=\linewidth]{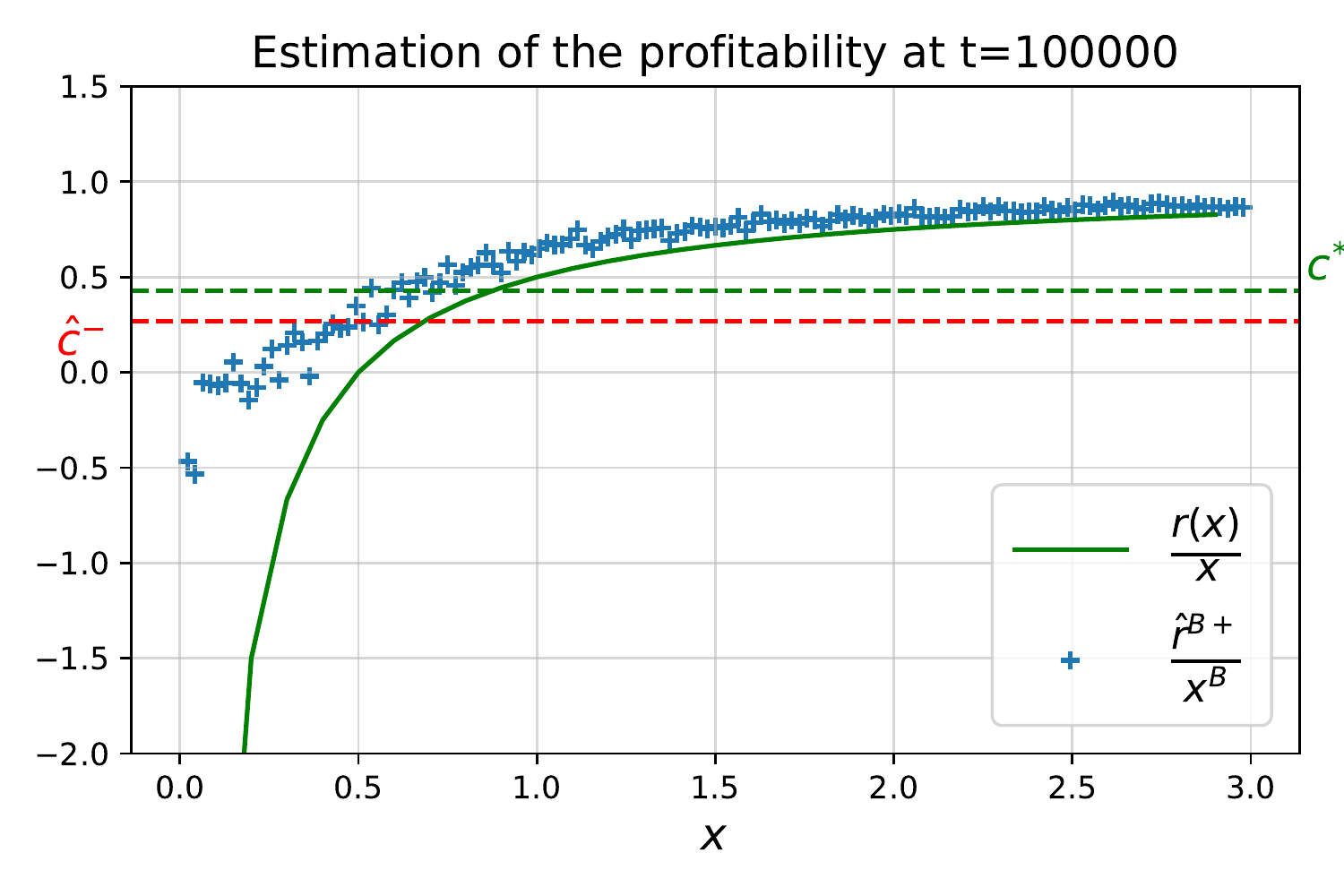}
  \caption{\label{fig:bandit1} Estimations of $\rew(\cdot)$ and $\profopt$ by \cref{alg:bandit} at $t=10^5$.
}
  \end{figure}
\end{minipage}%
\hspace{0.5cm}
\begin{minipage}{.48\linewidth}
\begin{figure}[H]
  \centering
\resizebox{\linewidth}{!}{\input{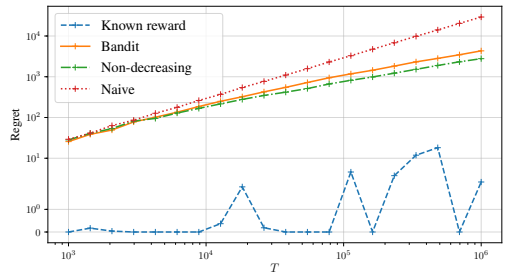}}
\caption{\label{fig:regret1}Evolution of regret with $T$ for different algorithms.}
  \end{figure}
\end{minipage}\end{adjustwidth}
\end{figure}

\cref{fig:regret1} shows in logarithmic scale the evolution of regret with the horizon $T$ for the different algorithms. The regret is averaged over $50$ runs (and $500$ runs for Known reward). 
\cref{alg:unknown-distribution} performs way better than its $\sqrt{T}$ bound, due to a better performance with margin conditions as mentioned above. Its regret is non-monotonic in $T$ here and is sometimes negative, in which case we round it to $0$ for clarity. This is only due to the variance in the Monte-Carlo estimates of the expected regret. For such a small regret, much more than $500$ runs are required to observe accurately the expected value. We unfortunately could not afford to simulate more runs for computational reasons.

\textit{Non-decreasing} corresponds to \cref{alg:monotone} and performs slightly better than \cref{alg:bandit} here. This remains a small improvement since \cref{alg:bandit} benefits from margin conditions as shown in \cref{SE:FastMargin}. Moreover, this improvement comes with additional computations.
Clearly, all these algorithms perform way better than the naive algorithm, which simply accepts every task. We indeed observe that learning occurs and the regret scales sublinearly with $T$ for all non-naive algorithms.

\subsection{Concave reward function}\label{sec:expeconcave}

This section now considers the case of the concave reward function $\rew(x) = -0.3x^2+x-0.2$. The profitability is not monotone here, making \cref{alg:monotone} unadapted.
We consider the same task distribution as in \cref{sec:expeaffine} and a Gaussian noise of variance $0.1$ for~$\varepsilon$.
Similarly to the affine case, \cref{fig:expe21,fig:expe22,fig:expe23} show the behaviors of the different algorithms, as well as the regret evolution. Here as well, the regret is averaged over $50$ runs (and $500$ runs for Known reward).

\begin{figure}[h]
\begin{subfigure}{.45\linewidth}
\centering
\resizebox{\linewidth}{!}{\includegraphics[scale=1]{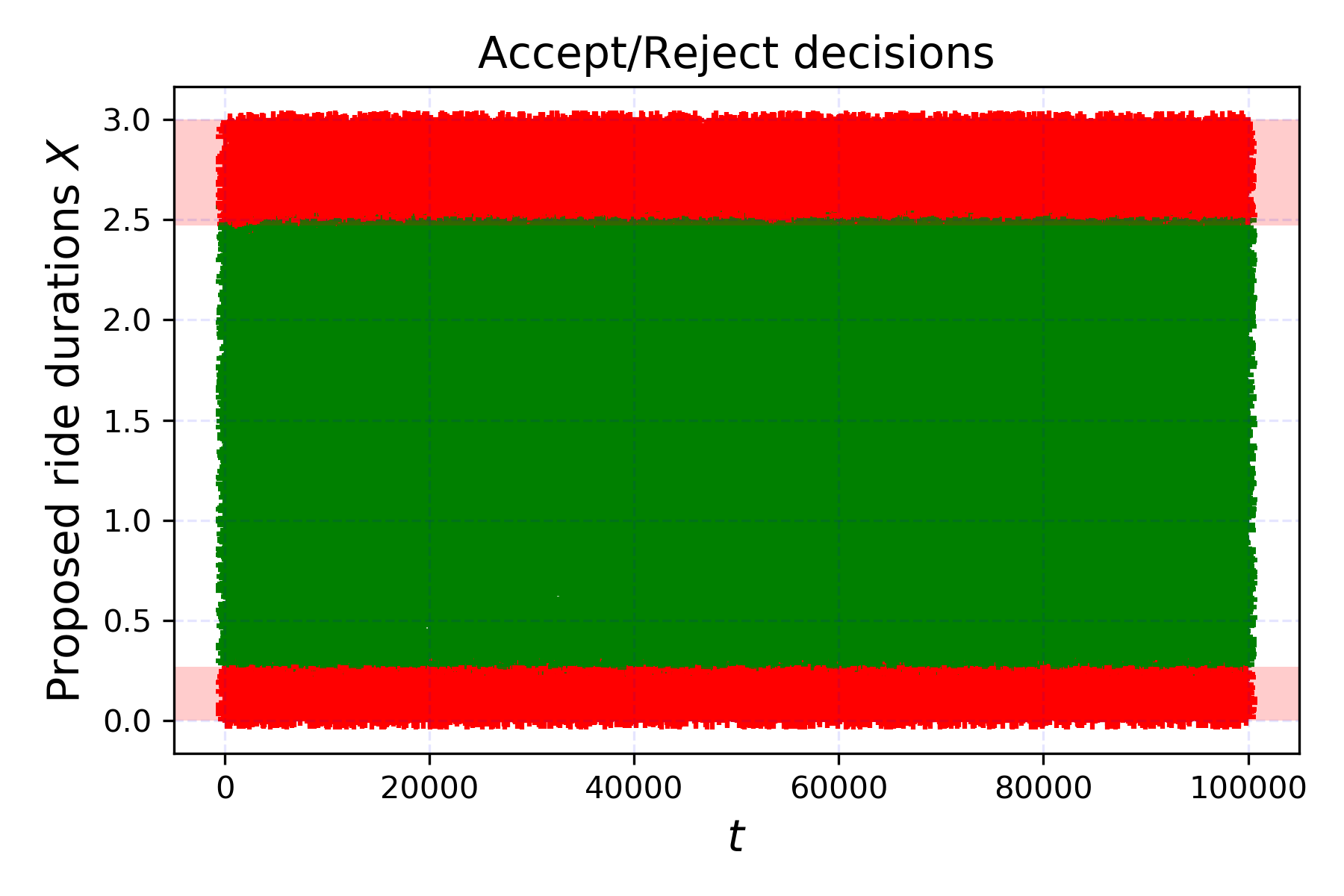}}
\caption{\cref{alg:unknown-distribution}: known reward.}
\end{subfigure}%
\hspace{0.2cm}
\begin{subfigure}{.45\linewidth}
\centering
\resizebox{\linewidth}{!}{\includegraphics[scale=1]{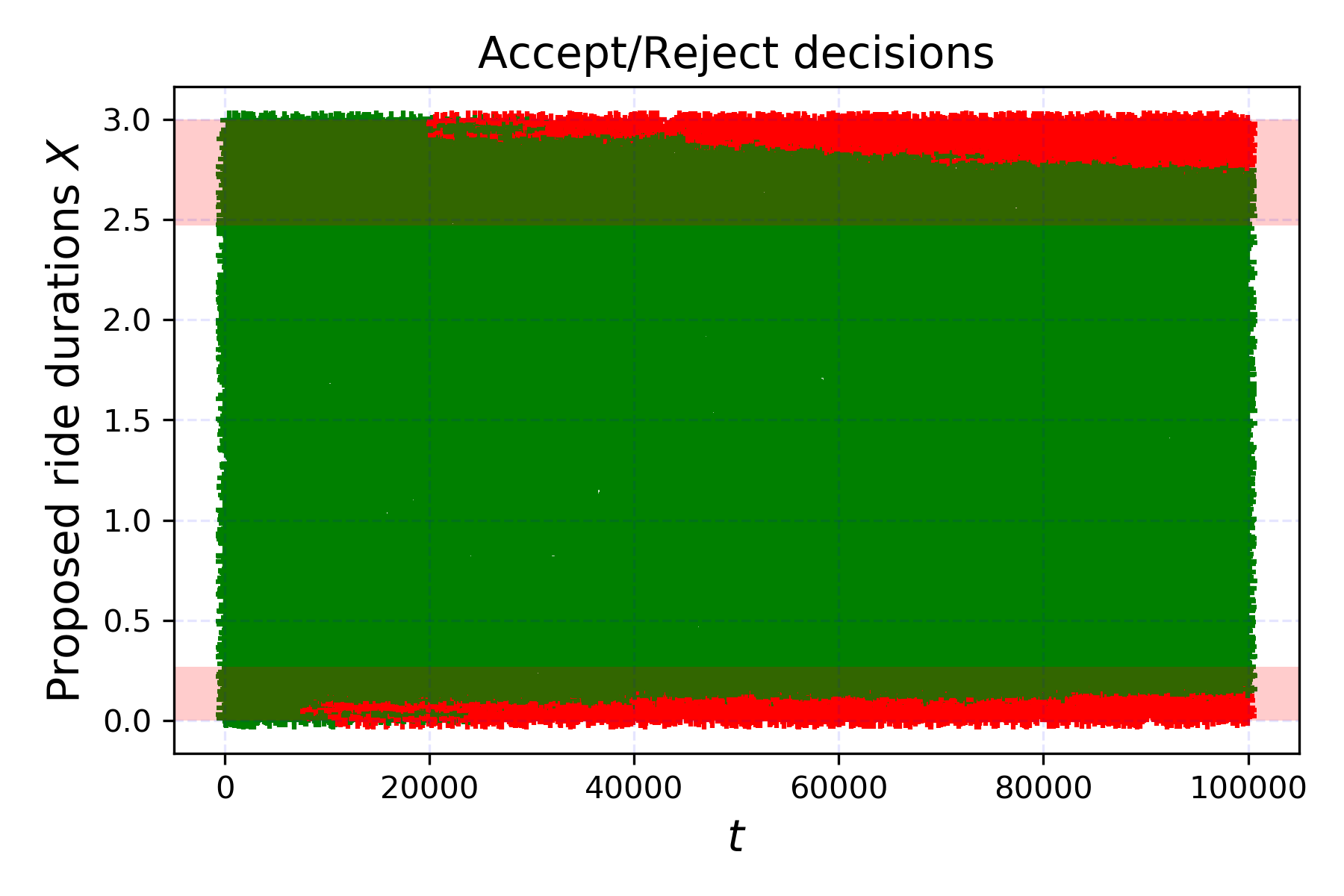}}
\caption{\cref{alg:bandit}: bandit setting.}
\end{subfigure}
\caption{\label{fig:expe21}Accept/Reject decisions on a single run, illustrated as in \cref{fig:decisions1}.}
\end{figure}

\begin{figure}[h]
\begin{adjustwidth}{-0.2cm}{-0.2cm}  
\centering
\begin{minipage}{.47\linewidth}
\begin{figure}[H]
  \includegraphics[width=\linewidth]{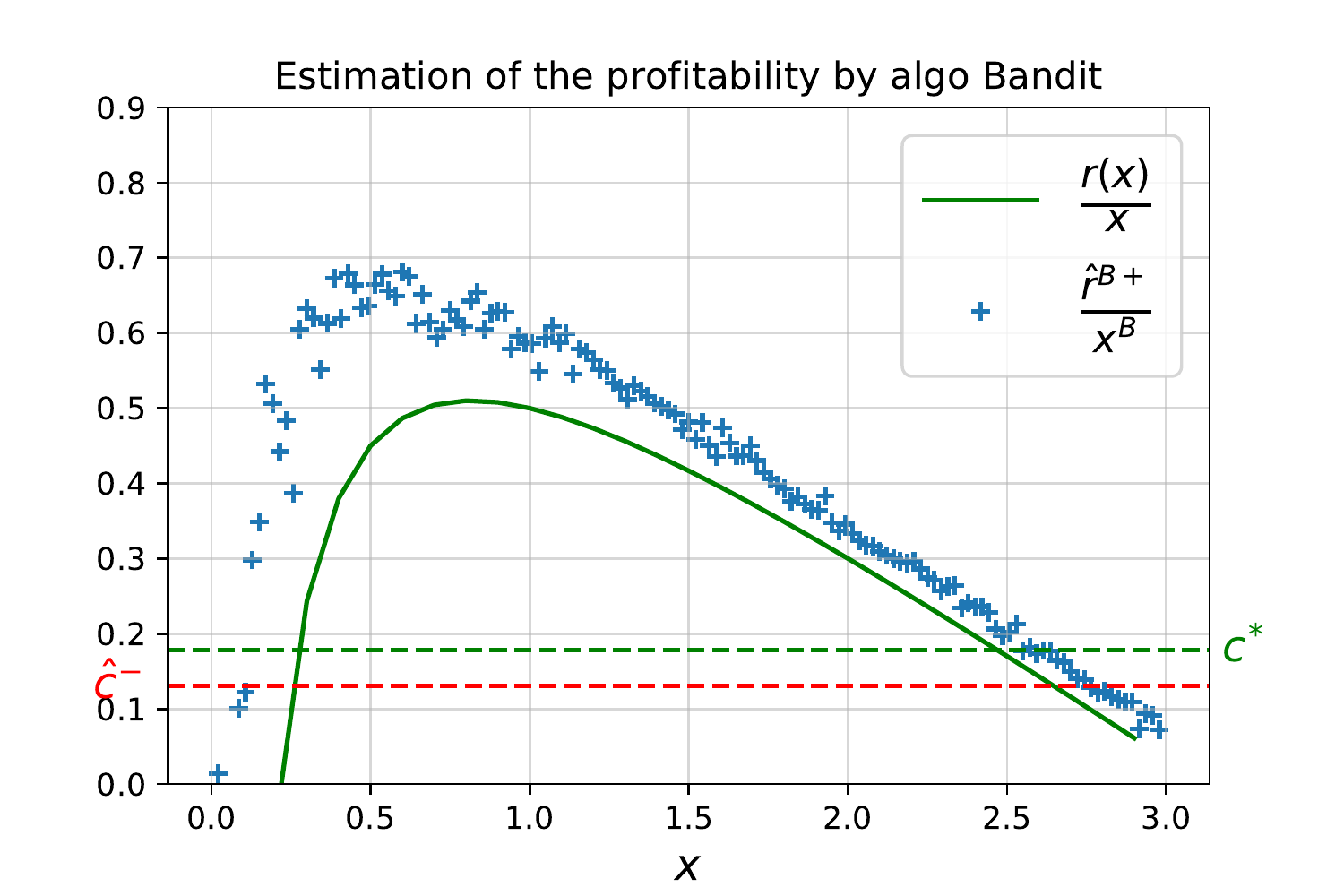}
  \caption{\label{fig:expe22}Estimation of profitability function by \cref{alg:bandit} after a time $t=10^5$.}
  \end{figure}
\end{minipage}%
\hspace{0.5cm}
\begin{minipage}{.47\linewidth}
\begin{figure}[H]
  \centering
  \resizebox{\linewidth}{!}{\input{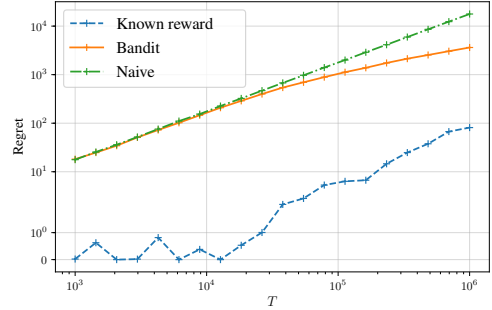}}
  \caption{\label{fig:expe23}Evolution of regret with time $T$.}
  \end{figure}
\end{minipage}\end{adjustwidth}
\end{figure}

The problem is now harder than in the previous affine case, since there are multiple suboptimal regions. 
\cref{alg:unknown-distribution} still performs very well for the same reasons, while \cref{alg:bandit} requires more time to detect suboptimal tasks, but still performs well in the end.
Similarly to \cref{fig:regret1}, the regret of Known reward is non-monotonic at first here. However it is increasing for larger values of $T$, as the variance becomes negligible with respect to the regret in this setting.
%
%
%

\section{Conclusion}

We introduced the problem of online allocation of time where an agent observes the duration of a task before accepting/declining it. Its main novelty and difficulty resides in the estimation and computation of a threshold, which depends both on the whole duration distribution and the reward function. 
After characterizing a baseline policy, we proposed computationally efficient learning algorithms with optimal regret guarantees, when the reward function is either known or unknown. Even faster learning rates are shown in \cref{sec:fast} under stronger assumptions on the duration distribution or the reward function.

The model proposed in this work is rather simple and aims at laying the foundations for possible extensions, that might be better adapted to specific applications. 
We  believe that \cref{alg:bandit} can be easily adapted to many natural extensions.
For instance, we here assume that the reward function only depends on the duration of the task, while it could also depend on additional covariates; using contextual bandits techniques with our algorithm should directly lead to an optimal solution in this case. Similarly, time is the only limited resource here, while extending this work to multiple limited resources seems possible using knapsack bandits techniques. It is also possible to extend it to several actions (instead of a single accept/decline action) by simultaneously estimating multiple rewards functions.

On the other hand, some extensions deserve careful considerations, as they lead to complex problems. In call centers for example, the agent might be able to complete several tasks simultaneously instead of a single one and characterizing the benchmark becomes much more intricate in this setting. 
Also, we consider stochastic durations/rewards here. Extending this work to the adversarial case is far from obvious. We yet believe this might be related to online knapsack problems and the use of a competitive analysis might be necessary in this case.
For ride-hailing applications, the spatial aspect is also crucial as a driver moves from a starting point to his destination during each ride. Adding a state variable to the model might thus be interesting here.

Finally, we believe this work only aims at proposing a first, general and simple model for online learning problems of time allocation, which can be extended to many settings, depending on the considered application.

\section*{Acknowledgements}

E. Boursier was supported by an AMX scholarship. V. Perchet acknowledges support from the French National Research Agency
(ANR) under grant number \#ANR-19-CE23-0026 as well as the support grant, as well as from
the grant “Investissements d’Avenir” (LabEx Ecodec/ANR-11-LABX-0047).
This research project received partial support from the COST action GAMENET.
M. Scarsini is a member of INdAM-GNAMPA.
His work was partially supported by the INdAM-GNAMPA 2020 Project ``Random walks on random games'' and the Italian MIUR PRIN 2017 Project ALGADIMAR ``Algorithms, Games, and Digital Markets.'' 
He gratefully acknowledges the kind hospitality of Ecole Normale Sup\'erieure Paris-Saclay, where this project started.

\newpage
\bibliographystyle{plainnat}
\bibliography{bibuber}

\newpage
\appendix

\section*{List of symbols}
\begin{longtable}{p{.13\textwidth} p{.82\textwidth}}
$a_i$ & decision concerning $i$-th task\\
$\bin_{j}$ & $j$-th bin\\
$\profopt$ & unique root of the function $\opzero$\\
$\prof_n$ & unique root of $\opzero_n$\\
$\profemp_n$ & estimate of $\profopt$\\
$\profemp_n^-$ & lower estimate of $\profopt$\\
$\tilde{\prof}$ & profitability threshold associated to $\rew_{\mathcal{E}}$\\
$\ubrv$ & upper bound for $\rv_i$\\
$\ubrew$ & upper bound for $\rew$\\
$\lbrew$ & lower bound for $\rew$\\
$\mathcal{E}$ & subset of sub-profitable tasks, introduced in \cref{lem:decline}\\
$\binsize$ & size of bins\\
$\mathcal{H}_t$ & history at time $t$\\
$K$ & $\card(\supp(X_1))$\\
$K_n$ & $\card\left\{ \rv_{i} \mid 1 \leq i \leq n\right\}$\\
$\holdcst$ & constant of the H\"older class, defined in \cref{def:holder}\\
$\numbin$ & number of bins\\
$n_\bin$ & number of received task proposals in the bin $\bin$\\
$\countbin_\bin$ & $\sum_{i=1}^n \one(\rv_i\in \bin \text{ and } a_i=1)$\\
$\profset$ & reward per time unit function, defined in \cref{eq:profit-thresh}\\
$\profset_n$ & empirical counterpart of $\profset$, defined in \cref{eq:thresh}\\
$q_\eps$ & $\E \bracks*{X\one(\rew(\rv)\geq (\profopt + \eps)\rv)}$, defined in \cref{eq:q-epsilon}\\
$\rew$ & reward function\\
$\rewest_n$ & estimate of $\rew$\\
$\rew(\rv_i)$ & $\expect[Y_{i}]$\\
$\rewest_n^{\bin}$ & $\rewest_n(\duration)\ \forall \duration\in \bin$, defined in \cref{eq:r-n-B}\\
$\rewest_n^+$ & upper estimate of $\rew$\\
$\tilde{\rew}_n^\bin$ & defined in \cref{eq:r-tilde}\\
$\rew_{\mathcal{E}}(x)$ & $\rew(x)\one(x \not\in \mathcal{E})$, introduced in \cref{lem:decline}\\
$\reg$ & regret, defined in \cref{eq:regret}\\
$\thresh$ & threshold\\
$\threshopt$ & threshold such that $\rew(\threshopt)/\threshopt = \profopt$\\
$\thresh_{n}$ & $\min \threshset_{n}$\\
$\exprv_i$ & idling time after $i$-th task\\
$\threshset_{n}$ & set of all potentially optimal thresholds, defined in \cref{eq:thresh-set}\\
$\timeT$ & problem horizon\\
$t_n$ & time at which the $n$-th task is proposed\\
$\mathcal{T}_n$ & amount of time after $n$ proposals\\
$\reward_\pi$ & expected reward, defined in \cref{eq:defnreward}\\
$\payv$ & value function\\
$w$ & bound on the value function\\
$x^\bin$ & $\min \left\{ x \in B \right\}$\\
$\rv_i$ &  duration of the $i$-th task\\
$Y_i$ & reward of the $i$-th task\\
$\alpha$ & parameter of the margin condition\\
$\holdexp$ & exponent of the H\"older class, defined in \cref{def:holder}\\
$\profrew_n$ & $\frac{\pr \E[\rew(\rv_n)\one(\rew(\rv_n)\geq\prof_n\rv_n)]}{1+\pr\E[\rv_n\one(\rew(\rv_n)\geq \prof_n\rv_n)]}$, defined in \cref{eq:gamma-n}\\
$\delta$ & probability bound, introduced in \cref{prop:prof_estimate}\\
$\Delta_{\min}$ & $\inf\{  \profopt \duration - \rew(\duration) \,\vert\,\rew(\duration) < \profopt \duration \}$, introduced in \cref{th:FastFinite}\\
$\noise_i$ & subgaussian noise\\
$\confprofset_{n}$ & defined in \cref{eq:confprofset}\\
$\confrew_{n-1}$ & defined in \cref{eq:eta-n-1}\\
$\totdur$ & total number of accepted task proposals\\
$\kappa$ & universal constant\\
$\pr$ & arrival rate of Poisson process\\
$\confprof_{n-1}$ & defined in \cref{eq:xi-n-1}\\
$\pi$ & policy\\
$\pi^*$ & threshold policy\\
$\sigma^2$ & proxy of the subgaussian noise\\
$\opzero$ & function $
\prof \mapsto \pr \E\left[\left(\rew(\rv) - \prof \rv \right)_+\right] - \prof$, defined in \cref{eq:opzero}\\
$\opzero_n$ &  empirical counterpart of $\opzero$, defined in \cref{eq:opzero_n}\\
\end{longtable}

\section*{Appendix}

\section{Fast rates}\label{sec:fast}
In this section, we investigate several cases where the general slow rate convergence  can be improved. 
We use two  assumptions about the distribution of task durations $\rv_i$ and one about the reward function~$\rew$.
For the sake of clarity, the proofs of all the results from this section are deferred to \cref{proof:fast}.

\subsection{Finite support}\label{SE:FastFinite}

The first assumption that can be used to obtain fast rates of convergence is the so-called ``finite support assumption''.
Under this assumption, the learning algorithm is quite simple and it consists in keeping an under-estimation  $\profemp_n^-$ of $\profopt$ and  an upper-estimation $\rewest^+_n(\duration)$ of $\rew(\duration)$ for all different values of $\duration$, whose number is denoted by  $K= \card(\supp(X_1))$. 

The finite algorithm is then based on \cref{alg:bandit}, with the difference that each bin corresponds to a single value of the support of the distribution: $B_j = \{x_j\}$, where $x_j$ is the $j$-th element in $\supp(X_1)$.
The estimates $\profemp_n$ and $\rewest_n$ are then computed similarly to \cref{alg:bandit}, but their uncertainty bounds are tighter. This yields the following optimistic estimates:
\begin{gather*}
\rewest^+_n(\duration) =\rewest_n(\duration) + \sigma\sqrt{\frac{\ln(K/\delta)}{2\countbin_{\bin}}} \qquad \text{for } \bin = \lbrace x \rbrace \text{ and } \\ \profemp^-_n= \profemp_n - 2\pr\sqrt{\sigma^2+\frac{(\ubrew-\lbrew)^2}{4}}\sqrt{\frac{\ln(1/\delta)}{n}}-\pr\sigma\sqrt{\frac{K}{2n}} - 8\pr\frac{K(D-E)}{n}.
\end{gather*}
The task $\rv_{n+1}$ is then accepted if and only if $\rewest^+_n(\rv_{n+1}) \geq \profemp^-_n\rv_{n+1}$. 

\begin{theorem}
\label{th:FastFinite}
Assume that the support of $\rv$ has cardinality $K$. 
Then, taking $\delta = 1/\timeT$, the finite algorithm's regret scales as
\[
\reg(\timeT) \leq \kappa_2  (1+\pr\ubrv)\max(\sigma , \ubrew-\lbrew)\sqrt{K\timeT\ln(\timeT)},
\]
where $\kappa_2 $ is some universal constant independent of any problem parameter. 

Moreover,  if  $\Delta_{\min} \coloneqq \inf\{  \profopt \duration - \rew(\duration) \,\vert\,\rew(\duration) < \profopt \duration \}$, then 
\[
\reg(\timeT) \leq \kappa_3 (1+\pr^2\ubrv^2)\max(\sigma , \ubrew-\lbrew)^2 K\frac{\ln\timeT}{\Delta_{\min}},
\]
for another universal constant $\kappa_3$.
\end{theorem}

\begin{remark}
The algorithm complexity scales with $K$ in this particular setting. The size $K$ of the support is assumed to be known for simplicity. 
For an unknown $K$, this algorithm can be adapted by using the observed value $K_n \coloneqq \card\left\{ \rv_{i} \mid 1 \leq i \leq n\right\}$ instead of $K$ and restarting every time $K_n$ increases. 
Such an algorithm  yields a similar performance, if not better in practice.
\end{remark}

\subsection{Margin condition}\label{SE:FastMargin}
The finite support assumption gives very fast learning rates but is quite strong. 
It is possible to study some ``intermediate'' regime between the general H\"older case and the finite support assumption, thanks to the margin condition recalled below. Intuitively, this condition states that slightly suboptimal tasks have a small probability of being proposed, thus limiting the regret from declining them.

\begin{definition}
The distribution of task durations $\rv_{i}$ satisfies the \emph{margin condition} with parameter $\alpha \geq 0$ if there exists $\kappa_0 >0$ such that for every $\varepsilon >0$,
\[
\proba\big( \profopt\rv - \varepsilon \leq \rew(\rv) < \profopt \rv \big) \leq   \kappa_0 \varepsilon^\alpha.
\]
\end{definition}

If the margin condition $\alpha$ is small enough (\ie the problem is not too easy and $\alpha <1$), then we can apply the very same algorithm as in the general case when $\rew$ is $(\holdcst,\holdexp)$-H\"older. 
We also assume that the distribution of task durations has a positive density lower-bounded by $\underline{\kappa}>0$ and upper bounded by $\overline{\kappa}>0$; with a slight abuse of notation, we call this being Lebesgue-equivalent.

\begin{theorem}\label{TH:Margin}
If $\rew$ is $(L,\beta)$-H\"older, the distribution of $\rv$ satisfies the margin condition with (unknown) parameter $\alpha \in [0, 1)$, and is Lebesgue-equivalent, then the regret of \cref{alg:bandit} (with same $\delta$ and $M$ as in \cref{th:regret-bandit}) scales as
\[
\reg(\timeT) \leq \kappa_4 \sigma^{1+\alpha}\left(\pr\ubrv\right)^{1+\alpha}\holdcst^{\frac{1+\alpha}{2\beta+1}}
(\pr\timeT+1)^{1-\frac{\holdexp}{2\holdexp+1}(1+\alpha)}\log(\pr\timeT+1)^{\frac{1+\alpha}{2}},
\]
where $\kappa_4$ is a constant depending only on $\underline{\kappa}$ and $\kappa_0$ defined in the margin condition assumption.
\end{theorem}

\begin{remark}
The proof follows the lines of Theorem~4.1 in \citet{PerRig:AS2013}. The term $\log(\pr\timeT+1)$  can be removed with a refined analysis using decreasing confidence levels $\delta_n = 1/n^2$ instead of a fixed $\delta$.\\
For larger margin conditions $\alpha \geq 1$, improved rates are still possible using an adaptive binning instead. The analysis should also follow the lines of \citet{PerRig:AS2013}.
\end{remark}

\subsection{Monotone profitability function}
\label{SE:FastMonotone}
We consider in this section an additional assumption on the profitability function $\duration\mapsto\rew(\duration)/\duration$, motivated by practical considerations. We assume that this function is non-decreasing (as the impact of some fixed cost declines with the task duration) -- the analysis would follow the same line for a non-increasing profitability function (to model diminishing returns for instance). 
A sufficient condition is convexity of the reward function $\rew$  with $\rew(0) \leq 0$. It is not difficult to see that a non-decreasing profitability function ensures the existence of  a threshold\footnote{For non-continuous reward, we only have that $\lim_{\substack{x \to \threshopt \\ x<\threshopt}} \frac{\rew(x)}{x} \leq \profopt \leq \lim_{\substack{x \to \threshopt \\ x>\threshopt}} \frac{\rew(x)}{x}$, but the following remains valid.} $\threshopt\in[0,\ubrv]$, such that
$ \rew(\threshopt)/\threshopt = \profopt$. 
Moreover, by monotonicity, it is  optimal to accept a task with duration $\duration$ if and only if $\duration\geq \threshopt$. 
Hence the problem reduces to learning this optimal threshold based on the reward gathered in the first tasks. 
In particular, we consider plug-in strategies that accept any task whose duration is large enough (and decline the other ones). 
The reward per time unit of these strategies is then a function of the threshold~$\thresh$, defined by
\begin{equation}
\label{eq:profit-thresh}
\profset(\thresh) = \frac{ \pr\E[\rew(\rv)\one(\rv\geq s)]}{1+\pr\E[\rv \one(\rv\geq s)]}.
\end{equation}

Our policy uses and updates a threshold $\thresh_n$  for all rounds.
Let $\numstage = 2(\pr\timeT+1)$ be an upper bound on the total number of received tasks with high probability.
As the objective is to find the optimum of the function $\profset(\cdot)$, we introduce its empirical counterparts defined, at the end of stage $n$, by
\begin{align}
\label{eq:thresh}
\profset_n(\thresh) = \frac{\frac{\pr}{n}\sum_{i=1}^{n} \noisedrew_i\one(\rv_i\geq\thresh)}{1+\frac{\pr}{n}\sum_{i=1}^{n} \rv_i\one(\rv_i\geq\thresh)}.
\end{align}
Given some confidence level $\delta\in(0,1]$ -- to be chosen later --, we also introduce the error term
\begin{equation}\label{eq:confprofset}
\confprofset_n = \left(\sqrt{\sigma^2 + \frac{(\ubrew - \lbrew)^2}{4}} + \frac{\ubrew-\lbrew}{\sqrt{2}}(\pr\ubrv+2)\right)\sqrt{\frac{\ln\left(\frac{2(\numstage+1)}{\delta}\right)}{n-1}} + \pr\frac{\ubrew-\lbrew}{n},
\end{equation}
in the estimation of $\profset(\cdot)$ by $\profset_n(\cdot)$, see \cref{lem:profset_estimate} in the Appendix. 

The policy starts with  a first threshold  $\thresh_1=0$. After observing a task, the first estimate $\profset_1(\cdot)$ is computed as in \cref{eq:thresh} and we proceed inductively. 
After $n-1$ tasks, given the estimate $\profset_{n-1}(\cdot)$ on $[\thresh_{n-1},\ubrv]$, we compute a lower estimate  of $\profset(\threshopt)$ by $\max_\thresh \profset_n(\thresh) - \confprofset_n$ (for notational convenience, we also introduce $\threshopt_{n-1} = \argmax_{\thresh\in[\thresh_{n-1},\ubrv]} \profset_{n-1}(\thresh)$).  A threshold is then potentially optimal if its associated reward per time unit (plus the error term) 
$\profset_n(\thresh) + \confprofset_n$ exceeds this lower estimate. 
Mathematically,  the set of all potentially optimal thresholds is  therefore:
\begin{equation}
\label{eq:thresh-set}
\threshset_{n} = \braces*{\thresh\in[\thresh_{n-1},\ubrv]\,\bigg\vert\, (\profset_{n-1}(\thresh)-\profset_{n-1}(\threshopt_{n-1}))\parens*{\frac{1}{\pr}+\frac{1}{n-1}\sum_{i=1}^{n-1} \rv_i\one(\rv_i\geq\thresh)}
+2\confprofset_{n-1} \geq 0 }.
\end{equation}
We then define the new threshold to be used in the next stage as  $\thresh_{n} = \min \threshset_{n}$. The rationale behind this choice is that it provides information for all the (potentially optimal) thresholds in $\threshset_{n}$.
The pseudo-code is given in \cref{alg:monotone} below.

We can finally state the main result when profitability is monotone. Regret scales as $\sqrt{T}$ independently of the regularity of the reward function $\rew$.  

\begin{theorem}\label{th:regret-profitability-function}
Assume the profitability function $\duration\mapsto \rew(\duration)/\duration$ is non-decreasing. Then, the regret of \cref{alg:monotone} (taking $\delta = 1/\timeT^2$) scales as
\begin{align}
\reg(\timeT) \leq \kappa_5 (\sigma+\ubrew-\lbrew)(1+\pr\ubrv)\sqrt{(\pr\timeT+1)\ln(\pr\timeT)},
\end{align}
where $\kappa_5$ is a constant independent from any problem parameter.
\end{theorem}
\begin{algorithm}[H]
\DontPrintSemicolon
\SetKwInOut{Input}{input}
\SetKw{From}{from}
\Input{$\delta$, $\pr$, $\ubrv$, $\ubrew$, $\sigma$, $\lbrew$, $\timeT$}
	$s_1=0$;\ $\numstage = 2\pr\timeT+1$;\ $n=0$\;
	\While{$\sum_{i=1}^n \exprv_i + \rv_i\one(\rv_i\geq \thresh_i)<\timeT$}{
		$n=n+1$\;
		Wait $S_n$ and observe $X_n$\;
			\lIf{$\rv_n \geq \thresh_n$}{accept task and receive reward $\noisedrew_n$}
			\lElse{reject task}
			Compute for $\thresh\in [\thresh_n,\ubrv]$, \hspace{1cm} $
\profset_n(\thresh) = \frac{\pr\frac{1}{n}\sum_{i=1}^{n} \noisedrew_i\one(\rv_i\geq\thresh)}{1+\pr\frac{1}{n}\sum_{i=1}^{n} \rv_i\one(\rv_i\geq\thresh)}$\;
Compute $\threshopt_{n} = \argmax_{\thresh\in[\thresh_{n},\ubrv]} \profset_{n}(\thresh)$\;
			Let $\threshset_{n+1} = \braces*{\thresh\in[\thresh_{n},\ubrv]\,\vert\, (\profset_{n}(\thresh)-\profset_{n}(\threshopt_{n}))\left(\frac{1}{\pr}+\frac{1}{n}\sum_{i=1}^{n} \rv_i\one(\rv_i\geq\thresh)\right)
			+2\confprofset_{n}  \geq 0 }$\;
			Set $\thresh_{n+1} = \min \threshset_{n+1}$\;}
\caption{\label{alg:monotone}Non-decreasing profitability algorithm}
\end{algorithm}
\begin{remark}
The function $\profset_n(\thresh)$ is piecewise constant in \cref{alg:monotone} and thus only requires to be computed at all the points $\rv_i$.
The algorithm then requires to store the complete history of tasks and has a complexity of order $n$ when receiving the $n$-th proposition. \\
However, a trick similar to \cref{rem:computation1} leads to improved complexity. The algorithm actually requires to compute $\profset_n(\thresh)$ only for all $s \in \threshset_n$. Only tasks of length in $\threshset_n$ then have to be individually stored and the computation of all ``interesting'' values of $\profset_n$ scales with the number of $\rv_i$ in $\threshset_n$. This yields a sublinear complexity in $n$ under regularity conditions on the distribution of $\rv$.
\end{remark}

\subsection{Lower bounds}\label{subsec:lower}

All the algorithms we propose are actually ``optimal'' for their respective class of problems. 
Here, optimality must be understood in the minimax sense, as a function of the horizon $\timeT$ and up to $\mathrm{poly}\log(\timeT)$ terms. 
More formally, a minimax lower-bound for a given class of problems (say, $(\holdcst,\holdexp)$-H\"older reward functions with bandit noisy feedback) states that, for any given algorithm, there always exists a specific problem instance such that $R(\timeT) \geq \kappa' \timeT^\gamma$ for some number $\gamma \in [0,1]$ and $\kappa'$ is independent of $\timeT$. 
An algorithm whose regret never increases faster than $\kappa'' \timeT^\gamma \log^{\gamma'}(\timeT)$, for some constants $\kappa''$ and $\gamma'$ independent of $\timeT$ is then \emph{minimax optimal}.

\begin{proposition}\label{th:lower}
Our algorithms are all minimax optimal for their respective class of problems.
\end{proposition}
The proof is rather immediate and postponed to \cref{app:lowerproof}. Indeed, the agent is facing a problem more complicated than a contextual (one-armed) bandit problem because $\profopt$ is unknown at first. 
On the other hand, if we assume that it is given to the agent, the problem is actually a one-armed contextual bandit problem. As a consequence, lower bounds for this problem are also lower-bounds for the agent's problem; as lower and upper bounds match (up to $\mathrm{poly}\log$ terms), the devised algorithms are optimal.

\section{Fast implementation using search trees}\label{sec:fastcomp}

First recall that the  complexity of computing the $n$-th decision of \cref{alg:unknown-distribution} is of order $n$ in worst cases, while \cref{alg:bandit} scales with $\numbin$, which is of order $T^{\frac{1}{2\holdexp+1}}$. The computational complexities of these algorithms at the $n$-th round can actually be respectively improved to $\log(n)$ and $\log(\numbin)$, while keeping the same space complexity, by the use of augmented balanced binary search trees data structures, \eg augmented red black trees \citep[][Chapters 12-14]{CorLeiRivSte:MIT2009}.
Such kind of trees stores nodes with keys in an increasing order. The search, insert and deletion function can thus all be computed in a time $\log(n)$, where $n$ is the number of nodes in the tree. 
An augmented structure allows to store additional information at each node, which is here used to compute the root of $\opzero_n$.

\paragraph{Fast \cref{alg:unknown-distribution}.} For \cref{alg:unknown-distribution}, each task $i$ is stored as a node with the key $\rew(\rv_i)/\rv_i$, with the additional piece of information $(\rew(\rv_i), \rv_i)$. Using augmented trees, it is possible to also store in each node $i$ the sum of $\rew(\rv_j)$ and $\rv_j$ for all the nodes $j$ in its right subtree, and similarly for its left subtree with the same complexity.

The main step of \cref{alg:unknown-distribution} is the computation of $\prof_n$, which can now be computed in a time $\log(n)$, thanks to the described structure. Indeed, note that
\begin{equation*}
\opzero_n\left(\frac{\rew(\rv_i)}{\rv_i}\right) = \frac{\pr}{n} \left(\sum_{j:\ \frac{\rew(\rv_j)}{\rv_j} > \frac{\rew(\rv_i)}{\rv_i}} \rew(\rv_j) - \frac{\rew(\rv_i)}{\rv_i}\sum_{j:\ \frac{\rew(\rv_j)}{\rv_j} > \frac{\rew(\rv_i)}{\rv_i}} \rv_j\right) - \frac{\rew(\rv_i)}{\rv_i}.
\end{equation*}

As the right subtree of the root exactly corresponds to the tasks $j$ with a larger profitability than the root, $\opzero_n(c)$ can be computed in time $1$ where $c$ is the profitability of the root node, using the additional stored information. 
As $\opzero_n$ is decreasing, the algorithm continues in the left subtree if $\opzero_n(c)<0$ and in the right subtree otherwise. Evaluating $\opzero_n$ at the key of the following node also takes a time $1$ using the different stored information. The algorithm then searches the tree in a time $\log(n)$ to find the task $i^*$ with the largest profitability such that $\opzero_n\left(\frac{\rew(\rv_{i^*})}{\rv_{i^*}}\right)>0$.

Note that $\opzero_n$ is piecewise linear with a partition given by the profitability of the received tasks. In particular, it is linear on $[\frac{\rew(\rv_{i^*})}{\rv_{i^*}}, \prof_n]$ and $\prof_n$ is given by
\begin{equation*}
\prof_n = \frac{\pr\sum_{j:\ \frac{\rew(\rv_j)}{\rv_j} > \frac{\rew(\rv_{i^*})}{\rv_{i^*}}} \rew(\rv_j)}{n + \pr \sum_{j:\ \frac{\rew(\rv_j)}{\rv_j} > \frac{\rew(\rv_{i^*})}{\rv_{i^*}}} \rv_j}.
\end{equation*}

\medskip

Recall that it was suggested in \cref{rem:computation1} to remove all tasks with profitability outside 
\[
\bracks*{c_n - \pr (\ubrew-\lbrew) \sqrt{\frac{\ln(1/\delta)}{2n}}, \ c_n + \pr (\ubrew-\lbrew) \sqrt{\frac{\ln(1/\delta)}{2n}}}
\]
and to only store the sum of rewards and durations for those above $c_n + \pr (\ubrew-\lbrew) \sqrt{\frac{\ln(1/\delta)}{2n}}$. Removing all nodes above or below some threshold in a search tree is called the split operation and can be done in time $\log(n)$ with the structures used here. Thanks to this, keeping a reduced space complexity is possible without additional computational cost.

\paragraph{Fast \cref{alg:bandit}.} For \cref{alg:bandit}, each node represents a bin with the key $(\frac{\tilde{\rew}_n^\bin}{x^\bin}, \bin)$ and the additional information $(n_{\bin} \, \tilde{\rew}_n^\bin, n_\bin \, x^\bin)$. $\bin$ is also required in the key to handle the case of bins with same profitabilities.

Here again at each node, we also store the sum of the additional information of the left and right subtrees. Note that at each time step, the algorithm updates $\tilde{\rew}_n^\bin$ for the bin of the received task. This is done by deleting the node with the old value\footnote{The algorithm also requires to store a list to find $\tilde{\rew}_n^\bin/x^\bin$ in time $1$ for any bin $\bin$.} of $\tilde{\rew}_n^\bin/x^\bin$ and insert a node with the new value.

The computation of $\profemp_n$ then follows the lines of the computation of $\prof_n$ in \cref{alg:unknown-distribution} above.

\medskip

Unfortunately, search trees do not seem adapted to \cref{alg:monotone} since maximizing the function $\profset_n(s)$ requires to browse the whole tree here.

\section{Proofs of \cref{se:model}}

\begin{proof}[Proof of \cref{prop:dynamic-programming}]
For $h>0$, let $B(t,h)$ denote the event ``the agent is proposed exactly one task in the interval $[\timet,\timet+h]$.''
Since the probability of receiving more than one task in this interval is $\smalloh(h)$, one has for $t < T$
\begin{align*}
\payv(\timet) &= \proba(B(t,h))\E[\max(\rew(\rv) + \payv(\timet + \rv),\payv(\timet+h))\,\vert\, B(t,h)] + (1-\proba(B(t,h)))\payv(t+h) + \smalloh(h)\\
&=(\pr h + \smalloh(h))\E[\max(\rew(\rv) + \payv(\timet + \rv),\payv(\timet+h))\,\vert\, B(t,h)] + (1-\pr h + o(h))\payv(\timet+h) + \smalloh(h). 
\end{align*}
By definition of the value function, the optimal strategy accepts a task $X$ at time $t$ if and only if $r(X) + \payv(t+X) \geq \payv(t)$. 
Hence
\begin{equation}
\label{eq:diff-v}
\frac{\payv(\timet+h)-\payv(\timet)}{h} = -(\pr + \smalloh(1))\E[(\rew(\rv) + \payv(\timet + \rv)-\payv(\timet+h))_+\, \vert\, B(t,h)] + \smalloh(1).
\end{equation}
Letting $h\to 0$ one has
$
\payv'(\timet) = -\pr\E[\left(\rew(\rv)+\payv(\timet+\rv)-\payv(\timet)\right)_+]$.

Since $(\timet,\payv)\mapsto -\pr\E[\left(\rew(\rv)+\payv(\timet+\rv)-\payv(\timet)\right)_+]$ is uniformly Lipschitz-continuous in $\payv$, uniqueness of the solution follows the same lines as the proof of the Picard-Lindel\"of (a.k.a. Cauchy-Lipschitz) theorem. 
Hence, $\payv$ is  the unique solution of the dynamic programming  \cref{eq:dynprog}.

Note that $\payw$ is actually the unique solution of the same equation without the boundary effect
\begin{equation}
\label{eq:wdyn}	
\begin{cases}
        \payw'(\timet) = -\pr \E\left[\left(\rew(\rv) + \payw(\timet + \rv) - \payw(\timet)\right)_+\right] \, &\text{ for all } \timet \in \R,\\
       \payw(\timeT) = 0. &{}
\end{cases}
\end{equation}

Similarly to the argument above, $\payw$ is the value function of the optimal strategy of an alternative program, defined as the original program, with the difference that, if the agent ends at time $t>T$, then she suffers an additional loss $(t-T) \profopt$ in her reward.

The optimal strategy does not only maximize the value function at time $t=0$, but actually maximizes it for any time $t \in [0, T]$.
By definition, any strategy earns less in the alternative program than in the original program. 
By optimality of the strategy giving the value $\payv$, this yields $\payv(t) \geq \payw(t)$ for any time $t$.

By translation, $\payw(\cdot - \ubrv)$ is also the value of the optimal strategy in the alternative program, but with horizon $T+\ubrv$. 
As the length of any task is at most $\ubrv$, following the optimal strategy in the original program of horizon $T$ and rejecting all tasks proposed between $T$ and $T+\ubrv$   yields exactly the value $\payv$ in this \emph{delayed} alternative program;  thus $\payw(t-\ubrv) \geq \payv(t)$ by optimality. 
\end{proof}

\begin{proof}[Proof of \cref{thm:baseline}]
Recall that the strategies considered in the proof of \cref{prop:dynamic-programming} accept a task $X$ at time $t$ if and only if $r(X) + v(t+X) \geq v(t)$ where $v$ is its value function. 
Thus, the optimal strategy in the alternative program described in the proof of \cref{prop:dynamic-programming} accepts a task $X$ if and only if $r(X) \geq \profopt X$.
Moreover, the cumulative reward of this strategy in the original program is larger than $\payw(0)$. 
The relation between $\payv$ and $\payw$ given by  \cref{prop:dynamic-programming} then yields the result.
\end{proof}

\section{Proofs of \cref{subsec:unknown-distrib}}

%
%

\begin{proof}[Proof of \cref{prop:prof_estimate}]
Let $\eps>0$. One has,
\begin{align*}
\prob(\prof_n-\profopt > \eps) &= \prob(\opzero_n(\prof_n) < \opzero_n(\profopt +\eps)) \text{ since } \opzero_n \text{ is decreasing}\\
&=  \prob(0 < \opzero_n(\profopt +\eps)) \text{ since } \prof_n \text{ is the root of } \opzero_n\\
&\leq \prob(\eps < \opzero_n(\profopt +\eps)-\opzero(\profopt +\eps)) \text{ since } \opzero(\profopt+\eps)\leq - \eps\\
&\leq \exp\left(\frac{-2n\eps^2}{\pr^2 (\ubrew-\lbrew)^2}\right),
\end{align*}
where the last inequality follows from Hoeffding's inequality.
\end{proof}

\begin{proposition}\label{prop:proprew}
For all $n\geq 1$, the following  bound holds
\begin{align*}
(\profopt - \profrew_n)\E\left[\rv_n\one(\rew(\rv_n)\geq \prof_n\rv_n) + \exprv_n\right]\ \leq \frac{\pr(\ubrew-\lbrew)\ubrv}{2}\sqrt{\frac{\pi}{2n}}.
\end{align*}
\end{proposition}
\begin{proof}[Proof of \cref{prop:proprew}]
First note that the definition of $\profopt$ implies
\begin{align*}
\profopt = \frac{\pr\E[\rew(\rv)\one(\rew(\rv)\geq \profopt\rv)]}{1+ \pr\E[\rv\one(\rew(\rv)\geq \profopt\rv)]}.
\end{align*}
Let us denote $\nume$ the numerator and $\den$ the denominator of the above expression.
First decompose
\begin{align*}
\one(\rew(\rv)\geq \prof_n\rv) = \one(\rew(\rv)\geq \profopt\rv) + \one(\profopt\rv> \rew(\rv)\geq \prof_n\rv) - \one(\prof_n\rv>\rew(\rv)\geq \profopt\rv).
\end{align*}
Denote 
\begin{align*}
\quant = \pr\E\left[\rv\Big(\one(\profopt\rv> \rew(\rv)\geq \prof_n\rv) - \one(\prof_n\rv>\rew(\rv)\geq \profopt\rv)\Big)\right].
\end{align*}
One has
\begin{align*}
\profrew_n
&= \frac{\nume + \pr\E[\rew(\rv)(\one(\profopt\rv> \rew(\rv)\geq \prof_n\rv) - \one(\prof_n\rv>\rew(\rv)\geq \profopt\rv))]}{\den + \pr\E[\rv(\one(\profopt\rv> \rew(\rv)\geq \prof_n\rv) - \one(\prof_n\rv>\rew(\rv)\geq \profopt\rv))]}\\
&\geq  \frac{\nume + \pr\E[\prof_n\rv(\one(\profopt\rv>\rew(\rv)\geq \prof_n\rv) - \one(\prof_n\rv>\rew(\rv)\geq \profopt\rv))] }{\den + \quant}\\
& = \profopt + \frac{\E\left[(\prof_n - \profopt)\quant\right]}{\den + \quant} \\
&\geq \profopt - \frac{\pr\ubrv\E[( \prof_n - \profopt )_+]}{1+\pr\E\left[\rv_n\one(\rew(\rv_n)\geq \prof_n\rv_n)\right]}\\
&\geq \profopt - \frac{\pr(\ubrew-\lbrew)\ubrv}{2}\sqrt{\frac{\pi}{2n}}\frac{1}{\E\left[\rv_n\one(\rew(\rv_n)\geq \prof_n\rv_n) + \exprv_n\right]},
\end{align*}
where the last inequality follows from \cref{prop:prof_estimate}.
\end{proof}

\begin{proof}[Proof of \cref{th:regret-fullinfo}]
From \cref{eq:regret-} one has
\begin{align*}
\reg(\timeT)
&= \E\left[\sum_{n=1}^{\totdur} (\profopt-\profrew_n)\E\left[\rv_n\one(\rew(\rv_n)\geq \prof_n\rv_n)+ \exprv_n\right]\right] + \profopt\left(\timeT-\E\left[\sum_{n=1}^{\totdur} \rv_n\one(\rew(\rv_n)\geq \prof_n\rv_n)+ \exprv_n\right]\right)\\
&\leq \E\left[\sum_{n=1}^{\totdur} (\profopt-\profrew_n)\E\left[\rv_n\one(\rew(\rv_n)\geq \prof_n\rv_n)+ \exprv_n\right]\right]\\
&\leq \frac{\pr(\ubrew-\lbrew)\ubrv}{2}\sqrt{\frac{\pi}{2}}\E\left[\sum_{n=1}^{\totdur} \sqrt{\frac{1}{n}}\right] \text{ by \cref{prop:proprew}}\\
&\leq \pr(\ubrew-\lbrew)\ubrv\sqrt{\frac{\pi}{2}}\E\left[\sqrt{\totdur}\right] \text{ by Wald's formula.}
\end{align*}
It remains to control the expected number of tasks observed $\E \totdur$. First remark that
\begin{align}\label{eq:ub-totdur}
\totdur -1 \leq \min\{n\geq 1 \,\vert\,\sum_{i=1}^n \exprv_i>\timeT\} -1 =   \sup\{ n\geq 1 \,\vert\,\sum_{i=1}^n \exprv_i\leq \timeT\}
\end{align}
and that the law of the right-hand side of \cref{eq:ub-totdur}  is Poisson of parameter $\pr \timeT$. So we get that  $\E \totdur \leq \pr\timeT +1$ and thus $\E \sqrt{\totdur} \leq \sqrt{\pr \timeT +1}$ by Jensen's inequality.
\end{proof}

\newpage 
\section{Proofs of \cref{subsec:bandit}} \label{proof:bandit}


\begin{proof}[Proof of \cref{lem:rewest}]
It follows from Hoeffding's inequality, and the fact that $\rew$ is H\"older continuous.
\end{proof}

\begin{proof}[Proof of \cref{lem:decline}]
The proof of this crucial lemma is actually straightforward, and follows from the fact that  $\profopt \geq 0$ and thus  $\left(\rew_\mathcal{E}(\rv)- \profopt \rv \right)_+=\left(\rew(\rv)- \profopt \rv \right)_+$.
\end{proof}

\begin{proof}[Proof of \cref{prop:giveaname0}]
Let us, in the first step, assume that all the first $N$ tasks have been accepted, so that $(\rv_i,\noisedrew_i)$  are i.i.d., and $\E [\noisedrew_i | \rv_i] = \rew(\rv_i)$.  We define
\begin{align*}
\opzeroest_n : \prof \mapsto \pr\E[\left(\rewest_n(\rv)-\prof x^{\bin(\rv)}\right)_+] - \prof,
\end{align*}
where ${\bin(\rv)}$ is the bin corresponding to $\rv$. Hence for all $\prof\geq 0$, $\E[\opzeroemp_n(\prof)] = \opzeroest_n(\prof)$.

Second, remark that for all $\prof \geq 0$ 
\begin{align}
\vert \opzeroest_n(\prof) - \opzero(\prof)\vert 
&= \pr \vert \E[(\rewest_n(\rv)-\prof x^{\bin(\rv)})_+ - (\rew(\rv)-\prof \rv)_+]\vert\\
&\leq   \pr\E[\vert\rewest_n(\rv)-\rew(\rv)\vert] + \pr \prof h\\
&\leq   \pr\E[\vert\rewest_n(\rv)-\rew(\rv)\vert] + \pr^2 \ubrew h.
\label{eq:opzero-binsize}
\end{align}
The last inequality is obtained by noting that $\profopt \leq \pr \ubrew$, and thus we only consider $\prof \leq \pr \ubrew$.

Third, for a fixed $\prof \geq 0$, $\opzeroemp_n(\prof)$ as a function of $(\noisedrew_1,\dots,\noisedrew_n)$,  is $\frac{\pr}{n}$-Lipschitz with respect to each variable.
Hence,
\begin{align*}
\prob(\profemp_n-\profopt > \eps) 
&= \prob(\opzeroemp_n(\profemp_n) < \opzeroemp_n(\profopt +\eps)) \text{ since } \opzeroemp_n \text{ is decreasing}\\
&=  \prob(0 < \opzeroemp_n(\profopt +\eps)) \text{ since } \profemp_i \text{ is the root of } \opzeroemp_i\\
&= \prob(-\opzeroest_n(\profopt +\eps) < \opzeroemp_n(\profopt +\eps)-\opzeroest_n(\profopt +\eps))\\
&\leq \proba(-\opzero(\profopt+\eps)- \pr \E[\vert\rewest_n(\rv)-\rew(\rv)\vert] - \pr^2 \ubrew h < \opzeroemp_n(\profopt +\eps)-\opzeroest_n(\profopt +\eps)) \text{ by } \cref{eq:opzero-binsize}\\
&\leq \proba(\eps-  \pr\E[\vert\rewest_n(\rv)-\rew(\rv)\vert] - \pr^2 \ubrew h < \opzeroemp_n(\profopt +\eps)-\opzeroest_n(\profopt +\eps))\text{ since } \opzero(\profopt+\eps)\leq - \eps\\
&\leq \exp\left(\frac{-n(\eps- \pr\E[\vert\rewest_n(\rv)-\rew(\rv)\vert]- \pr^2 \ubrew h)^2}{4\pr^2(\sigma^2+\frac{(\ubrew-\lbrew)^2}{4})}\right) \text{ by McDiarmid's inequality}.
\end{align*}
The last inequality uses McDiarmid's inequality for subgaussian variables \cite[Theorem~1]{kontorovich2014concentration}.
We conclude using \citet[Corollary 11.2]{GyoKohKrzWal:Springer2002}, which yields that $\E[\vert\rewest_n(\rv)-\rew(\rv)\vert] \leq \kappa \max(\sigma, \frac{D-E}{2}) \sqrt{\frac{\log(n)+1}{hn}} + \sqrt{8}\frac{Lh^\beta}{2^\beta}$.

\bigskip

Now, we focus on the case where some of the first $N$ tasks have been declined by the agent. We are going to prove the result by induction. Consider the event where, on the first $n-1$ tasks, it always happened that $\rewest^+_i(\cdot) \geq \rew(\cdot)$ and $\profemp_i^- \leq \profopt$. As a consequence, on this event, only sub-optimal tasks are declined. 
Using \cref{lem:decline}, this yields that the optimal value per time step associated to $\rew(\cdot)$ and to $\rew(\cdot)\one( \tilde{\rew}_n(\cdot) > 0)$ are equal.

As a consequence, the precedent arguments can be applied to the following function
\begin{equation*}
\tilde{\opzero}_n : \prof \mapsto \lambda \expect\left[ (\tilde{\rew}_n(X) - c x^{B(X)}) \right]
\end{equation*}
instead of $\bar{\opzero}_n$. 
Although \citet[Corollary~11.2]{GyoKohKrzWal:Springer2002} requires the function $\rew$ to be $(\holdexp,\holdcst)$-H\"older, it also holds here as $\rew(\cdot)\one( \tilde{\rew}_n(\cdot) > 0)$ is $(\holdexp,\holdcst)$-H\"older on every interval of the subdivision $\bin_1, \ldots, \bin_M$.

\end{proof}

Note that in the proof above, we always consider $\eps > \sqrt{8}\pr\frac{\holdcst\binsize^{\holdexp}}{2^\holdexp}+ \pr^2 \ubrew h$. Thus, by removing the term $\sqrt{8}\pr\frac{\holdcst\binsize^{\holdexp}}{2^\holdexp}+ \pr^2 \ubrew h$ in $\confprof_{n-1}$, all tasks with profitability above $\profopt + \eps$ remain observed. With a slight modification of the last argument, it can still be shown that $\opzeroemp_n(\profemp_n^-) < 0$ with high probability. 
While this has no influence on the theoretical order of the regret, it is used in the experiments as it yields a significant practical improvement.

\begin{proof}[Proof of \cref{th:regret-bandit}]
Recall that upon the $n$-th call, the agent computes $\rewest_{n-1}$ and $\profemp_{n-1}$ and accepts the task $\rv_{n}\in \bin$ if and only if  $\rewest^+_{n-1}(\rv_n) \geq  \profemp_{n-1}^- x^{\bin(\rv_n)}$. 
Let us denote by $\refuse(n)$ this event so that the regret can be decomposed into
\begin{align*}
\reg(\timeT) 
&= \profopt \timeT - \E\left[\sum_{n=1}^{\totdur} \rew(\rv_n)\one(\refuse(n))\right]\end{align*}
where 
\begin{align*}
\totdur \coloneqq \min\{N\in\N \,\vert\, \sum_{n=1}^N \exprv_n + \rv_n\one\left(\refuse(n)\right)> \timeT\}.
\end{align*}
The regret can then be rewritten as
\begin{align*}
\reg(\timeT) &= \profopt \timeT - \E\left[\sum_{n=1}^{\totdur} \E\left[\rew(\rv_n)\one(\refuse(n))\right]\right] \text{ by Wald's equation}\\
&= \profopt \timeT - \E\left[\sum_{n=1}^{\totdur} \profrewest_n \E\left[\rv_n\one(\refuse(n)) + \exprv_n\right]\right],
\end{align*}
where the reward per time unit of task $n$ is 
\begin{align*}
\profrewest_n \coloneqq \frac{\pr \E[\rew(\rv_n)\one(\refuse(n))]}{1+\pr\E[\rv_n\one(\refuse(n)]}.
\end{align*}

We further decompose the regret into

\begin{align*}
\reg(\timeT)
&= \E\left[\sum_{n=1}^{\totdur} (\profopt-\profrewest_n)\E\left[\rv_n\one(\refuse(n)) + \exprv_n\right]\right] + \profopt\left(\timeT-\E\left[\sum_{n=1}^{\totdur} \rv_n\one(\refuse(n)) + \exprv_n\right]\right)\\
&\leq \E\left[\sum_{n=1}^{\totdur} (\profopt-\profrewest_n)\E\left[\rv_n\one(\refuse(n)) + \exprv_n\right]\right]\\
&\leq\E\left[\sum_{n=1}^{\totdur}  2\ubrv\confprof_{n-1} + 4n\ubrew\delta+ \profopt h+2\E[\confrew_{n-1}(\rv_n)] \right],
\end{align*}
where the last inequality is a consequence of the following \cref{prop:profrewest}. As a consequence, it only remains to  bound the last term
\begin{align*}
&  \E\left[\sum_{n=1}^{\totdur}\E[\confrew_{n-1}(\rv_n)]\right] \\
&= \E\left[\sum_{n=1}^{\totdur}\E\left[\sum_{j=1}^{\numbin}\confrew_{n-1}(\bin_j)\one(\rv_n\in\bin_j)\right]\right]\\
&= \E\left[\sum_{n=1}^{\totdur}\E\left[\sum_{j=1}^{\numbin}\sqrt{\sigma^2 + \frac{\holdcst^2}{4}\left(\frac{\ubrv}{\numbin}\right)^{2\beta}}\sqrt{\frac{\ln(\numbin/\delta)}{2\countbin_{\bin_j}(n-1)}}\one(\rv_n\in\bin_j)\right]+ \totdur L\left(\frac{\ubrv}{\numbin}\right)^\beta\right]\\ 
&= \sqrt{\sigma^2 + \frac{\holdcst^2}{4}\left(\frac{\ubrv}{\numbin}\right)^{2\beta}}\sqrt{\ln(\numbin/\delta)}\E\left[\sum_{j=1}^{\numbin}\sum_{n=1}^{\totdur}\frac{\one(\rv_n\in\bin_j)}{\sqrt{2\countbin_{\bin_j}(n-1)}}\right]+ \totdur L\left(\frac{\ubrv}{\numbin}\right)^\beta
\\
&\leq \sqrt{\sigma^2 + \frac{\holdcst^2}{4}\left(\frac{\ubrv}{\numbin}\right)^{2\beta}}\sqrt{\ln(\numbin/\delta)}\E\left[\sum_{j=1}^{\numbin}\sqrt{\countbin_{\bin_j}(\totdur-1)}\right]+ \totdur L\left(\frac{\ubrv}{\numbin}\right)^\beta\\
&\leq \sqrt{\sigma^2 + \frac{\holdcst^2}{4}\left(\frac{\ubrv}{\numbin}\right)^{2\beta}}\sqrt{\ln(\numbin/\delta)}\E\left[\sqrt{\numbin \totdur}\right]+ \totdur L\left(\frac{\ubrv}{\numbin}\right)^\beta\\
&\leq \sqrt{\sigma^2 + \frac{\holdcst^2}{4}\left(\frac{\ubrv}{\numbin}\right)^{2\beta}}\sqrt{\ln(\numbin/\delta)}\sqrt{\numbin(\pr\timeT+1)}+ (\pr\timeT+1) L\left(\frac{\ubrv}{\numbin}\right)^\beta
\end{align*}
Hence, putting all things together, we get

\begin{align*}
\reg(\timeT)  & \leq 2\pr\ubrv\Big(4\sqrt{\sigma^2 + \frac{(\ubrew - \lbrew)^2}{4}}\sqrt{\ln(1/\delta)}\sqrt{\pr\timeT+1} + \kappa\max(\sigma, \frac{\ubrew-\lbrew}{2})\sqrt{\frac{\numbin}{\ubrv}\log(e(\pr\timeT+1))}\sqrt{\pr\timeT+1}\\
& \hspace{2cm}  +2\sqrt{2}\frac{\holdcst \ubrv^\beta}{(2\numbin)^\holdexp}(\pr\timeT+1) + \pr\ubrew \frac{\ubrv}{\numbin}\Big)\\
&\quad +4\ubrew(\pr\timeT+1)\delta + \profopt (\pr \timeT + 1) \frac{C}{M}\\
&\quad +2\sqrt{\sigma^2 + \frac{\holdcst^2}{4}\left(\frac{\ubrv}{\numbin}\right)^{2\beta}}\sqrt{\ln(\numbin/\delta)}\sqrt{\numbin(\pr\timeT+1)}+ 2(\pr\timeT+1) L\left(\frac{\ubrv}{\numbin}\right)^\beta\ .
\end{align*}
The result follows from the specific choices $\delta=\frac{1}{T^2}$ and $\numbin = \left\lceil \ubrv \holdcst^{\frac{2}{2\holdexp+1}}(\pr \timeT +1)^{\frac{1}{2\holdexp+1}} \right\rceil$.
\end{proof}

\begin{proposition}\label{prop:profrewest}
For all $n\geq 1$,
\begin{align*}
0 \leq (\profopt - \profrewest_n)\E\left[\rv_n\one(\refuse(n)) + \exprv_n\right] \leq  2\ubrv\confprof_{n-1} +4n\ubrew\delta+2\E[\confrew_{n-1}(\rv_n)]+ \profopt h.
\end{align*}
\end{proposition}

\begin{proof}
Recall that 
\begin{align*}
\profopt = \frac{\pr\E[\rew(\rv)\one(\rew(\rv)\geq \profopt \rv)]}{1+\pr\E[\rv\one(\rew(\rv)\geq \profopt \rv)]}= \frac{\nume}{\den}.
\end{align*}

We first decompose the indicator function in the numerator of $\profrewest_n$. Recall that the task $n$ is accepted (the event $\refuse(n)$ occurs) if $\rewest^+_{n-1} \geq \profemp^-_{n-1}x^{\bin(\rv_n)}$. The following holds
\begin{align*}
\one(\overbrace{\rewest^+_{n-1} \geq \profemp^-_{n-1}x^{\bin(\rv_n)}}^{\refuse(n)}) +  \one(\overbrace{\profopt\rv_n<\rewest^+_{n-1}(\rv_n) < \profemp^-_{n-1}x^{\bin(\rv_n)}}^{\mathcal{E}_2}) + \one(\overbrace{\rewest^+_{n-1}(\rv_n)\leq \profopt\rv_n\leq \rew(\rv_n)}^{\mathcal{E}_4})\\
= \one(\underbrace{\rew(\rv_n)\geq \profopt\rv_n}_{\mathcal{E}_0})
+ \one(\underbrace{\profopt\rv_n\geq\rewest^+_{n-1}(\rv_n)\geq \profemp^-_{n-1}x^{\bin(\rv_n)}}_{\mathcal{E}_1})+ \one(\underbrace{\rewest^+_{n-1}(\rv_n)> \profopt\rv_n> \rew(\rv_n)}_{\mathcal{E}_3}).
\end{align*}
To prove it, just notice the followings:
\begin{itemize}
\item $\refuse(n)\cap\mathcal{E}_4 = \mathcal{E}_0 \cap \mathcal{E}_1$,
\item $\mathcal{E}_2$ is disjoint with both $\refuse(n)$ and $\mathcal{E}_4$,
\item $\mathcal{E}_3$ is disjoint with both $\mathcal{E}_0$ and $\mathcal{E}_1$,
\item $\refuse(n) \cup \mathcal{E}_2 \cup \mathcal{E}_4 = \mathcal{E}_0 \cup \mathcal{E}_1 \cup \mathcal{E}_3$.
\end{itemize}
It then comes
\begin{align*}
\one(\refuse(n)) + \one(\mathcal{E}_2)+ \one(\mathcal{E}_4) & = \one(\refuse(n) \cup \mathcal{E}_2 \cup \mathcal{E}_4) + \one(\refuse(n)\cap\mathcal{E}_4) \\
& = \one(\mathcal{E}_0 \cup \mathcal{E}_1 \cup \mathcal{E}_3) + \one(\mathcal{E}_0 \cap \mathcal{E}_1) \\
& = \one(\mathcal{E}_0) + \one(\mathcal{E}_1)+ \one(\mathcal{E}_3).
\end{align*}
The first equality comes from the second point; the second from the first and last point; while the last equality comes from the third point.
This gives the following 
\begin{align*}
\one(\refuse(n)) =& \one(\mathcal{E}_0)+ \one(\mathcal{E}_1)-\one(\mathcal{E}_2)+ \one(\mathcal{E}_3)- \one(\mathcal{E}_4).
\end{align*}

The quantity of interest is then rewritten as:
\begin{align*}
(\profopt - \profrewest_n)\E\left[\rv_n\one(\refuse(n)) + \exprv_n\right] & = \profopt\E\left[\rv_n\one(\refuse(n)) + \exprv_n\right] - \expect\left[r(X_n) \one(\refuse(n)) \right]\\
& = \profopt\E\left[\rv_n\big(\one(\mathcal{E}_1)-\one(\mathcal{E}_2)+ \one(\mathcal{E}_3)- \one(\mathcal{E}_4)\big)\right] \\ &\phantom{=} - \expect\left[r(X_n) \big(\one(\mathcal{E}_1)-\one(\mathcal{E}_2)+ \one(\mathcal{E}_3)- \one(\mathcal{E}_4)\big) \right] \\
& \leq \profopt\E\left[\rv_n\big(\one(\mathcal{E}_1)+ \one(\mathcal{E}_3)\big)\right] - \expect\left[r(X_n) \big(\one(\mathcal{E}_1)-\one(\mathcal{E}_2)+ \one(\mathcal{E}_3)- \one(\mathcal{E}_4)\big) \right]
\end{align*}

Let us now bound the last four  terms. 
\begin{enumerate}
    \item  Recall that $\mathcal{E}_1=\{\profopt\rv_n\geq\rewest^+_{n-1}(\rv_n) \geq \profemp^-_{n-1}x^{\bin(\rv_n)}\}$,  so that
\begin{align*}
\E[\rew(\rv_n)\one(\mathcal{E}_1)]
=& \E[\rewest^+_{n-1}(\rv_n)\one(\mathcal{E}_1)] +\E[(\rew(\rv_n)-\rewest^+_{n-1}(\rv_n))\one(\mathcal{E}_1)] \\
\geq& \E[\profemp^-_{n-1} x^{\bin(\rv_n)}\one(\mathcal{E}_1)] -\E[\vert\rew(\rv_n)-\rewest^+_{n-1}(\rv_n)\vert\one(\mathcal{E}_1)] \\
\geq& \profopt\E[x^{\bin(\rv_n)}\one(\mathcal{E}_1)]  -\ubrv\E[\vert\profemp^-_{n-1}-\profopt\vert\one(\mathcal{E}_1)] -\E[\vert\rew(\rv_n)-\rewest^+_{n-1}(\rv_n)\vert\one(\mathcal{E}_1)] \\
\geq& \profopt\E[\rv_n\one(\mathcal{E}_1)]  -\ubrv\E[\vert\profemp^-_{n-1}-\profopt\vert\one(\mathcal{E}_1)] -\E[\vert\rew(\rv_n)-\rewest^+_{n-1}(\rv_n)\vert\one(\mathcal{E}_1)] + \profopt h \mathbb{P}(\mathcal{E}_1).
\end{align*}
\item $\mathcal{E}_2$ happens with probability at most $n\delta$ since $x^{\bin(\rv_n)}\leq \rv_n$ and using \cref{prop:giveaname0}, which upper bounds the second term by $\ubrew n\delta$.
\item Recall that $\mathcal{E}_3=\{\rewest_{n-1}^+(\rv_{n}) > \profopt\rv_n> \rew(\rv_n)\}$ so that the third term is bounded as
\begin{align*}
\E[\rew(\rv_n)\one(\mathcal{E}_3)]
&= \E[\rewest^+_{n-1}(\rv_{n})\one(\mathcal{E}_3)]
+ \E[(\rew(\rv_n)-\rewest^+_{n-1}(\rv_{n}))\one(\mathcal{E}_3)]\\
&\geq \profopt\E[\rv_{n}\one(\mathcal{E}_3)]  - \E[\vert \rew(\rv_n)-\rewest^+_{n-1}(\rv_{n})\vert\one(\mathcal{E}_3) ].
\end{align*}
\item $\mathcal{E}_4$ happens with probability at most $n\delta$ thanks to \cref{lem:rewest}, which upper bounds the fourth term by $\ubrew n\delta$.
\end{enumerate}
Putting everything together we get
\begin{align*}
    (\profopt - \profrewest_n)\E\left[\rv_n\one(\refuse(n)) + \exprv_n\right] \leq \E[\vert \rew(\rv_n)-\rewest^+_{n-1}(\rv_{n})\vert\one(\mathcal{E}_1\cup\mathcal{E}_3)] & + \ubrv\E[\vert\profemp^-_{n-1}-\profopt\vert\one(\mathcal{E}_1)] \\
    & + 2\ubrew n\delta + \profopt h.
\end{align*}
The result follows by noting that on $\mathcal{E}_1$, with probability at least $1-n\delta$, then $\profopt-2\confprof_{n-1} \leq \profemp^-_{n-1}$, and similarly for $\rewest^+_{n-1}-\rew$ on $\mathcal{E}_1 \cup \mathcal{E}_3$.
\end{proof}

\newpage
\section{Proofs of \cref{sec:fast}}\label{proof:fast}

\subsection{Proofs of \cref{SE:FastFinite,SE:FastMargin}}
\begin{proof}[Proof of  \cref{th:FastFinite}]
The proof of this theorem is almost a direct consequence of the proofs of \cref{prop:profrewest,th:regret-bandit}, it only requires a few tweaks. 

Similarly to \cref{lem:rewest}, it can be shown that $\vert \rewest_n(x) - \rew(x)\vert \leq \sigma \sqrt{\frac{\ln(K/\delta)}{2N{\lbrace x \rbrace}}}$ with probability at least $1-2N\delta$ for all $n\leq N$ and $x$. The uncertainty on $\profemp_n$ is shown similarly, except for the term $\expect[\vert \rewest_n(\rv_n) - \rew(\rv_n) \vert]$ which is bounded as follows:
\begin{align*}
\expect[\vert \rewest_n(\rv_n) - \rew(\rv_n) \vert] & = \expect\left[ \sum_x p(x) \expect[\vert \rewest_n(\rv_n) - \rew(\rv_n) \vert \mid N_{\lbrace x \rbrace}]  \right] \\
& \leq \expect\left[ \sum_x p(x) \min\left(\frac{\sigma}{2 \sqrt{N_{\lbrace x \rbrace}}}, D-E\right)\right] \quad \text{using Hoeffding's inequality} \\
& \leq \sum_x p(x) \left( \frac{\sigma}{\sqrt{2p(x) n}} + (D-E)e^{-\frac{p(x)n}{8}} \right) \\
& \leq \sigma\sqrt{\frac{K}{2n}} + \frac{8 K (D-E)}{n}.
\end{align*}

The Chernoff's bound $\mathbb{P}(N_{\lbrace x \rbrace} \leq \frac{p(x)}{2n}) \leq e^{-\frac{p(x)n}{8}}$ yields the second inequality; while the third one comes from Cauchy-Schwarz inequality and uses that $pe^{-\frac{pn}{8}} \leq \frac{8}{n}$.

Following the same arguments and as standard in multi-armed bandit, we basically need to compute the number of times tasks $x$ are incorrectly accepted. 
Consider  the event where, for all tasks, it holds $\rewest^+_n(x) \leq r(x) + 2(\ubrew-\lbrew)\sqrt{\frac{\ln(K/\delta)}{\countbin(\duration)}}$ (with the standard convention that $1/0=+\infty$) and that $\profemp^-_n \geq \profopt - 4\pr(\ubrew-\lbrew)\sqrt{K\frac{\ln(1/\delta)}{n}}$. 
Then, for $n$ such that $\rew(x) + 2 \confrew_n(x) < (\profopt - 2 \confprof_n)x$, $x$ stops being accepted. 
In particular, this yields for some constant $\kappa'$ that the total number of accepted tasks of duration $\duration$ (on this event) is smaller than
\begin{align*}
N_{\lbrace x \rbrace} \leq \kappa' \frac{K\pr^2(\ubrv^2 + (\ubrew-\lbrew)^2)+(\sigma^2+(D-E)^2)(1+\pr^2\ubrv^2)\ln(K/\delta)}{(\profopt x - \rew(x))^2} + \sqrt{K},
\end{align*}
which gives a contribution to the regret of the order of (up to a multiplicative constant)
$$
\frac{K\pr^2(\ubrv^2 + (\ubrew-\lbrew)^2)+(\sigma^2+(D-E)^2)(1+\pr^2\ubrv^2)\ln(K/\delta)}{\profopt x - \rew(x)} + \sqrt{K}(\profopt x - \rew(x)),
$$
We conclude as usual thanks to the choice of $\delta$ that ensures that the  contribution to the regret on the complimentary event is negligible.
\end{proof}

\begin{proof}[Proof of  \cref{TH:Margin}]

Similarly to the proof of \cref{th:regret-bandit}, denoting $\Delta(x) = \profopt x - \rew(x)$, the regret can be decomposed as
\begin{align}
\reg(\timeT) \leq & \E \sum_{n=1}^\totdur \sum_{j=1}^\numbin \one(x \in \bin_j)\Delta(x)\one(2\confrew_{n-1}(x) + 2\confprof_{n-1} \ubrv \geq \Delta(x) \geq 0) + 4n\ubrew\delta \nonumber \\
& \leq \E \sum_{n=1}^\totdur \sum_{j=1}^\numbin \one(x \in \bin_j)\Delta(x)\one(4\confrew_{n-1}(x) \ubrv \geq \Delta(x) \geq 0) + \one(x \in \bin_j)\Delta(x)\one(4\confprof_{n-1} \ubrv \geq \Delta(x) \geq 0)+ 4n\ubrew\delta \label{eq:margin0}
\end{align}

The contribution of the third term can be bounded similarly to \cref{th:regret-bandit} and is $\mathcal{O}(1)$.

Using the margin condition, the second term scales with $\sum_{n=1}^{\pr\timeT+1} (\ubrv\confprof_{n-1})^{1+\alpha}$, which is of order $$(\pr\ubrv \holdcst^{1-\frac{2}{2\holdexp+1}})^{1+\alpha} (\pr\timeT+1)^{1-\frac{\holdexp}{2\holdexp +1}(1+\alpha)}$$.

It now remains to bound the first term. It can be done using the analysis of  \citet{PerRig:AS2013}, which we sketch here.

The idea is to divide the bins into two categories for some constant $c_1$ scaling with $\sigma \ubrv^{\frac{2\holdexp+1}{2}}\holdcst \sqrt{\log(\pr\timeT+1)}$:
\begin{itemize}
\item \emph{well behaved} bins, for which $\exists x \in \bin, \Delta(x) \geq c_1 M^{-\holdexp}$,
\item \emph{ill behaved} bins, for which  $\forall x \in \bin, \Delta(x) < c_1 M^{-\holdexp}.$
\end{itemize}

The first term in \cref{eq:margin0} for ill behaved bins is directly bounded, using the margin condition, by a term scaling with
\begin{align*}
c_1^{1+\alpha} M^{-\holdexp(1+\alpha)}(\pr\timeT+1) \approx \sigma^{1+\alpha}\ubrv^{\frac{1+\alpha}{2}}\holdcst^{\frac{1+\alpha}{2\beta+1}} \log(\pr\timeT+1)^{\frac{1+\alpha}{2}}(\pr\timeT+1)^{1-\frac{\holdexp}{2\holdexp +1}(1+\alpha)}.
\end{align*}

Now denote $\mathcal{J}\subset \{1, \ldots, \numbin \}$ the set of well behaved bins and for any $j \in \mathcal{J}$, $\Delta_j = \max_{x \in \bin_j} \Delta(x)$. Using classical bandit arguments, the event $\{x \in \bin_j\}\cap\{4\confrew_{n-1}(x) \ubrv \geq \Delta(x) \geq 0\}$ only holds at most $\kappa'' (\sigma^2 + \holdcst^2 \left(\frac{\ubrv}{\numbin}\right)^2)\frac{\log(\pr\timeT)}{\Delta_j^2}$ times. 

Assuming w.l.o.g. that $\mathcal{J} = \left\{1, \ldots, j_1 \right\}$ and $\Delta_1 \leq \ldots \leq \Delta_{j_1}$, the margin condition can be leveraged to show that $\Delta_j \geq \left( \frac{\underline{\kappa} j}{\kappa_0 M} \right)^{\frac{1}{\alpha}}$. The contribution of well behaved bins then scales as
\begin{equation}\label{eq:margin1}
(\sigma^2 + \holdcst^2 \left(\frac{\ubrv}{\numbin}\right)^2)\sum_{j=1}^{j_1} \frac{\log(\pr\timeT+1)}{\Delta_j} \leq \sigma^2 \log(\pr\timeT+1)\left( \frac{j_2 M^{\holdexp}}{c_1} + \sum_{j=j_2+1}^{j_1} \left( \frac{j}{M} \right)^{-\frac{1}{\alpha}} \right),
\end{equation}
where $j_2$ is some integer of order $c_1^\alpha\numbin^{1-\alpha\holdexp}$ so that for any $j\leq j_2$, $c_1 M^{-\holdexp} \geq \left( \frac{\underline{\kappa} j}{\kappa_0 M} \right)^{\frac{1}{\alpha}}$.

The first term of \cref{eq:margin1} can then be bounded by \begin{gather*}(\sigma^2 + \holdcst^2 \left(\frac{\ubrv}{\numbin}\right)^2)c_1^{\alpha-1} \numbin^{1+\holdexp(1-\alpha)}\log(\pr\timeT+1)\lesssim 
\sigma^{1+\alpha} \ubrv^{\frac{1+\alpha}{2}}\holdcst^{\frac{1+\alpha}{2\beta+1}}(\pr\timeT+1)^{1-\frac{\holdexp}{2\holdexp+1}(1+\alpha)}\log(\pr\timeT+1)^{\frac{1+\alpha}{2}},
\end{gather*}
where $\lesssim$ means that the inequality holds up to some universal constant $\kappa'$.

The sum in \cref{eq:margin1} can be bounded as follows:
\begin{align*}
\sum_{j=j_2+1}^{j_1} \left( \frac{j}{M} \right)^{-\frac{1}{\alpha}} & \leq \int_{\frac{j_2}{\numbin}}^{1} x^{\frac{1}{\alpha}} \mathrm{d}x \\
& \lesssim \frac{c_1^{\alpha-1} M^{\holdexp(1-\alpha)}}{1-\alpha}.
\end{align*}

Finally, the first term of \cref{eq:margin0} scales as
$$ \sigma^{1+\alpha}\ubrv^{\frac{1+\alpha}{2}}\holdcst^{\frac{1+\alpha}{2\beta+1}}(\pr\timeT+1)^{1-\frac{\holdexp}{2\holdexp+1}(1+\alpha)}\log(\pr\timeT+1)^{\frac{1+\alpha}{2}},$$
which finally yields \cref{TH:Margin} when gathering everything.
\end{proof}
\newpage
\subsection{Proofs of  \cref{SE:FastMonotone}}



The following Lemma indicates that with arbitrarily high probability, our upper/lower estimations are correct


\begin{lemma}\label{lem:profset_estimate}
With a constant $\kappa_6$ independent from $n$, the events 
\begin{align*}
\vert\profset_n(\thresh)-\profset(\thresh)\vert \E[\rv\one(\rv\geq\thresh)+\frac{1}{\pr}]
\leq & \confprofset_n  + \frac{\kappa_6}{n}.
\end{align*}
hold with probability at least $1-{\delta}$, simultaneously for all $n\in\{1,\dots,\numstage\}$ and all $\thresh\in [\thresh_n,\ubrv]$.
\end{lemma}

\begin{proof}
Let
\begin{align}
\profset(\thresh) = \frac{ \pr\E[\rew(\rv)\one(\rv \geq s)]}{1+\pr\E[\rv \one(\rv\geq s)]}
\eqqcolon \frac{\nume(\thresh)}{\den(\thresh)}.
\end{align}
and
\begin{align}
\profset_n(\thresh) = \frac{\pr\frac{1}{n}\sum_{i=1}^{n} \noisedrew_i\one(\rv_i\geq\thresh)}{1+\pr\frac{1}{n}\sum_{i=1}^{n} \rv_i\one(\rv_i\geq\thresh)} \eqqcolon \frac{\nume_n(\thresh)}{\den_n(\thresh)}.
\end{align}
Let us rewrite
\begin{align*}
\frac{\den(\thresh)}{\pr} &= \frac{1}{\pr}+\thresh(1-\cdf(\thresh))+\int_{\thresh}^{\ubrv} 1-\cdf(\duration)d\duration \text{ and }\\
\frac{\den_n(\thresh)}{\pr} &= \frac{1}{\pr}+\thresh(1-\cdf_{n}(\thresh))+\int_{\thresh}^{\ubrv} 1-\cdf_{n}(\duration)d\duration,
\end{align*}
where $\cdf_{n}$ is the empirical distribution function $\cdf_{n}(\duration) \coloneqq \frac{1}{n}\sum_{i=1}^{n} \one(\rv_{i}\leq \duration)$.
Hence
\begin{align*}
    \frac{\vert \den(\thresh) - \den_n(\thresh)\vert}{\pr} 
    &\leq \thresh \vert \cdf(\thresh) - \cdf_{n}(\thresh) \vert + \int_{\thresh}^{\ubrv} \vert\cdf(\duration) - \cdf_{n}(\duration) \vert d\duration\\
    &\leq \thresh \vert \cdf(\thresh) - \cdf_{n}(\thresh) \vert + (\ubrv-\thresh)\max_{\duration\in[\thresh,\ubrv]} \vert\cdf(\duration) - \cdf_{n}(\duration) \vert\\
    &\leq \ubrv\max_{\duration\in[\thresh,\ubrv]} \vert\cdf(\duration) - \cdf_{n}(\duration) \vert.
\end{align*}
The Dvoretzky–Kiefer–Wolfowitz (DKW) inequality ensures that the event $\vert \den(\thresh) - \den_n(\thresh)\vert\leq \pr\ubrv\sqrt{\frac{1}{2n}\ln\left(\frac{2\numstage}{\delta}\right)},\; \forall \thresh\in[0,\ubrv]$, holds with probability at least $1-\frac{\delta}{\numstage}$.

Denote by $\thresh^1,\thresh^2,\dots,\thresh^{n}$ a permutation of observed task durations $\rv_1,\dots,\rv_{n}$ such that $\thresh^1\leq\thresh^2\leq\dots\leq\thresh^{n}$. Also define $\thresh^0=\thresh_n$ for completeness. Let $\thresh\in[\thresh^k,\thresh^{k+1})$, since $\lbrew \leq 0$, it comes
\begin{align*}
    \nume(\thresh) &= \pr\E[\rew(\rv)\one(\rv\geq\thresh)]\\
    &\leq \nume(\thresh^{k}) - \pr\lbrew (\cdf(\thresh) - \cdf(\thresh^k))\\
    &\leq \nume(\thresh^{k}) - \pr\lbrew(\cdf(\thresh^{k+1}) - \cdf(\thresh^k))\\
    & = \nume(\thresh^{k}) -  \pr\lbrew\Big(\cdf(\thresh^{k+1})-\cdf_{n}(\thresh^{k+1}) + \cdf_{n}(\thresh^{k+1}) - \cdf_{n}(\thresh^{k}) + \cdf_{n}(\thresh^{k}) - \cdf(\thresh^k)\Big)
\end{align*}
Hence,
\begin{align*}
    \nume(\thresh) \leq \nume(\thresh^{k}) - 2\pr\lbrew\sqrt{\frac{1}{2n}\ln\left(\frac{\numstage}{\delta}\right)} - \frac{\pr\lbrew}{n} \quad \forall k\in\{1,\dots, n\},
\end{align*}
holds as soon as the DKW inequality above holds.
We prove analogously for the event
\begin{align*}
    &\nume(\thresh) \geq \nume(\thresh^{k}) - 2\pr\ubrew\sqrt{\frac{1}{2n}\ln\left(\frac{\numstage}{\delta}\right)} - \frac{\pr\ubrew}{n} \; \forall k\in\{1,\dots,n\}.
\end{align*}

Recall that $\thresh^k = X_{i_k}$ for some $i_k$ for $k \in \{ 1, \ldots, n \}$. By Hoeffding's inequality with $(\rv_{i})_{i\neq i_k}$ as samples and a union bound on $\{0,\dots,n\}$, the event
\begin{align*}
   \vert \nume(\thresh^k)-\nume_n(\thresh^k)\vert \leq \pr\sqrt{\sigma^2+\frac{(\ubrew - \lbrew)^2}{4}}\sqrt{\frac{1}{2(n-1)}\ln\left(\frac{2\numstage (n+1)}{\delta}\right)}\quad \forall k \in \{0,\dots,n\}
\end{align*}
holds with probability at least $1-\frac{\delta}{\numstage}$.

Putting everything together with the fact that $\nume_n(\thresh) = \nume_n(\thresh^k)$ and $n\leq S$, one gets that the event
\begin{align*}
    \vert \nume_n(\thresh) - \nume(\thresh)\vert \leq \pr\left(\sqrt{\sigma^2+\frac{(\ubrew-\lbrew)^2}{4}} + \sqrt{2}(\ubrew-\lbrew) \right) \sqrt{\frac{\ln(2\numstage/\delta)}{n-1}}  + \frac{\pr(\ubrew - \lbrew)}{n}
\end{align*}
holds with probability at least $1-\frac{\delta}{\numstage}$ for all $\thresh\in [\thresh_n,\ubrv]$ if the DKW inequality also holds.

Furthermore,
\begin{align}
\left| \frac{\nume_n(\thresh)}{\den_n(\thresh)} - \frac{\nume(\thresh)}{\den(\thresh)} \right|
&\leq  \left|\nume_n(\thresh)\right|\left|\frac{\den(\thresh)-\den_n(\thresh)}{\den(\thresh)\den_n(\thresh)}\right| +  \frac{1}{\den(\thresh)}\left|\nume_n(\thresh)-\nume(\thresh)\right|\\
&\leq \frac{\big(\pr(\ubrew-\lbrew)+\vert \nume_n(\thresh) - \nume(\thresh)\vert\big)\left|\den(\thresh)-\den_n(\thresh)\right| +  \left|\nume_n(\thresh)-\nume(\thresh)\right|}{1+\pr\E[\rv\one(\rv\geq\thresh)]}.\label{eq:profset_estimate1}
\end{align}
Hence with probability at least $1-\frac{2\delta}{\numstage}$ one has for all $\thresh\in [\thresh_n,\ubrv]$,
\begin{align*}
    \vert\profset_n(\thresh)-\profset(\thresh)\vert \E[\rv\one(\rv\geq\thresh)+\frac{1}{\pr}]
    \leq & \pr\left(\sqrt{\sigma^2+\frac{(\ubrew-\lbrew)^2}{4}} + \frac{\ubrew-\lbrew}{\sqrt{2}}(\pr\ubrv+2) \right) \sqrt{\frac{\ln(2\numstage/\delta)}{n-1}} \\
    & + \pr\frac{\ubrew - \lbrew}{n} \\
    & + \pr \ubrv\left(\sqrt{\frac{\sigma^2}{2}+\frac{(\ubrew-\lbrew)^2}{8}} + \ubrew-\lbrew \right) \frac{\ln\left(\frac{2\numstage}{\delta}\right)}{n-1} \\
    & + \pr\ubrv\frac{\ubrew-\lbrew}{\sqrt{2}n^{3/2}}\sqrt{\ln\left(\frac{2\numstage}{\delta}\right)}.
\end{align*}

We conclude using a union bound over $\{1,\dots,\numstage\}$.
\end{proof}

The goal of the next lemma is to ensure that the sequence $(\thresh_n)_{1\leq n\leq \numstage}$ is indeed bellow the the optimal threshold $\threshopt$ with high probability.
\begin{lemma}\label{prop:lower-estimate-threshold}
With probability at least $1-\delta$, the events
\[
\thresh_n\leq \threshopt\quad  \text{ and } \quad 0\leq (\profset_n(\threshopt_n)-\profset_n(\threshopt))\left(\frac{1}{\pr}+\frac{1}{n}\sum_{i=1}^{n} \rv_i\one(\rv_i\geq\thresh)\right)\leq 2 \confprofset_n,
\]
 hold simultaneously in all stages $n\in\{1,\dots,\numstage\}$.
 \end{lemma}
 
 \begin{proof}
The proof is by induction on $n\geq 1$. For $n=1$, $\thresh_1 = 0 \leq \threshopt$. We prove in a similar manner to \cref{lem:profset_estimate}, that the events
\begin{equation}\label{eq:proofconfprofset1}
(\profset(\thresh) - \profset_n(\thresh))\left(\frac{1}{\pr}+\frac{1}{n}\sum_{i=1}^{n} \rv_i\one(\rv_i\geq\thresh)\right) \leq \confprofset_n \quad \forall \thresh \in [\thresh_n,\ubrv]
\end{equation}
simultaneously hold for all stages $n\in\{1,\dots,\numstage\}$, with probability at least $1-\delta$. The only difference with \cref{lem:profset_estimate} is that the term $\vert \nume_n(\thresh) - \nume(\thresh) \vert$ in \cref{eq:profset_estimate1} does not appear here, removing the $\frac{\kappa_6}{n}$ term in \cref{lem:profset_estimate}.
In particular, this implies the events
\begin{gather*}
(\profset(\threshopt) - \profset_1(\threshopt))\left(\frac{1}{\pr}+\rv_i\one(\rv_1\geq\thresh)\right) \leq \confprofset_1 \\ \text{ and } (\profset_1(\threshopt_1) - \profset(\threshopt_1))\left(\frac{1}{\pr}+ \rv_1\one(\rv_1\geq\thresh)\right)\leq \confprofset_1
\end{gather*}
Since $\profset(\threshopt_1)\leq \profset(\threshopt)$, one has that
\begin{align*}
(\profset_1(\threshopt_1)-\profset_1(\threshopt))\left(\frac{1}{\pr}+ \rv_i\one(\rv_1\geq\thresh)\right)\leq 2\confprofset_1
\end{align*}
holds when \cref{eq:proofconfprofset1} holds.

Now for any $n> 1$, the induction assumption implies that $\threshopt\in \threshset_n$, hence $\threshopt\geq \thresh_n$. The rest of the induction follows the steps of the base case $n=1$. Finally, one has $0\leq \profset_n(\threshopt_n) - \profset_n(\threshopt)$ by definition of $\threshopt_n$.
\end{proof}
 
The following proposition allows to control the regret in stage $n$.
\begin{proposition}\label{prop:regret-monotone}
For some constant $\kappa_7$ independent from $n$, with probability at least $1-\delta$, the events
\begin{align*}
(\profset(\threshopt)-\profset(\thresh_n))\E[\rv\one(\rv\geq\thresh_n)+\frac{1}{\pr}]
& \leq 4\confprofset_{n-1} + \frac{\kappa_7}{n}.
\end{align*}
hold simultaneously in all stages $n\in\{2,\dots,\numstage\}$.
\end{proposition}

\begin{proof}
With probability at least $1-\delta$, the following inequality hold uniformly for all stages $n\geq 2$ by \cref{lem:profset_estimate}:
\begin{align*}
(\profset(\threshopt)-\profset(\thresh_n))\E[\rv\one(\rv\geq\thresh_{n})+\frac{1}{\pr}]
&\leq (\profset(\threshopt) - \profset_{n-1}(\thresh_n))\E[\rv\one(\rv\geq\thresh_n)+\frac{1}{\pr}] + \confprofset_{n-1} + \frac{\kappa_6}{n}
\end{align*}

Now remark, using the notation of \cref{lem:profset_estimate} that, thanks to the DKW inequality and the fact that $\thresh_n \in \threshset_{n}$, under the same probability event,
\begin{align*}
(\profset_{n-1}(\threshopt_{n-1})-\profset_{n-1}(\thresh_{n})) \frac{\den(\thresh_n)}{\pr} & = (\profset_{n-1}(\threshopt_{n-1})-\profset_{n-1}(\thresh_{n})) \frac{\den_n(\thresh_n)}{\pr} + (\profset_{n-1}(\threshopt_{n-1})-\profset_{n-1}(\thresh_{n})) \frac{\den(\thresh_n)-\den_n(\thresh_n)}{\pr} \\
& \leq 2 \confprofset_{n-1}+ 2 \confprofset_{n-1}\ubrv \sqrt{\frac{\ln(\frac{4\numstage}{\delta})}{2(n-1)}}.
\end{align*}
The definition of $\threshset_n$ directly bounds the first term. For the second term, note that the definition of  $\threshset_n$ also bounds $\profset_{n-1}(\threshopt_{n-1})-\profset_{n-1}(\thresh_{n})$. Combined with the concentration on $\den_n-\den$, this bounds the second term. It then follows
\begin{align*}
(\profset(\threshopt)-\profset(\thresh_n))\frac{\den(\thresh_n)}{\pr} & \leq (\profset(\threshopt) - \profset_{n-1}(\threshopt_{n-1})) \E[\rv\one(\rv\geq\thresh_{n})+\frac{1}{\pr}] + 3\confprofset_{n-1} + \frac{\kappa_8}{n} \\
&\leq (\profset(\threshopt) - \profset_{n-1}(\threshopt))\E[\rv\one(\rv\geq\thresh_{n-1})+\frac{1}{\pr}] + 3\confprofset_{n-1} + \frac{\kappa_8}{n}  \\
&\leq 4\confprofset_{n-1} + \frac{\kappa_8+\kappa_6}{n} \text{ by \cref{lem:profset_estimate} and since }\threshopt\geq \thresh_{n-1} \text{ by \cref{prop:lower-estimate-threshold}}.
\end{align*}
\end{proof}

\begin{lemma}\label{lem:poisson}
For $\numstage = 2\pr\timeT+1$, we have
\begin{equation*}
\E\left[(\totdur - \numstage)_+\right] \leq 4.
\end{equation*}
\end{lemma}
\begin{proof}
Similarly to the proof of \cref{th:regret-fullinfo}, note that $\totdur-1$ is dominated by a random variable following a Poisson distribution of parameter $\pr T$. It now just remains to show that for any random variable $Z$ following a Poisson distribution of parameter $\pr$, it holds that $\E\left[(Z-2\pr)_+ \right] \leq 4$.

Thanks to \cite{canonne2017}, the cdf of $Z$ can be bounded as follows for any positive $x$:
\begin{align*}
\proba(Z \geq \pr + x) & \leq e^{-\frac{x^2}{2(\pr+x)}},
\end{align*}
which implies for $x=\pr + t$ with $t>0$:
\begin{align*}
\proba(Z-2\pr \geq t) & \leq e^{-\frac{(t+\pr)^2}{2(2\pr+t)}} \\
& \leq e^{-\frac{t+\pr}{4}}.
\end{align*}
From there, the expectation of $(Z-2\pr)_+$ can be directly bounded:
\begin{align*}
\E[(Z-2\pr)_+] & \leq \int_{0}^\infty e^{-\frac{t+\pr}{4}} \mathrm{d}t \\
& \leq 4.
\end{align*}
This concludes the proof.
\end{proof}

\begin{proof}[Proof of \cref{th:regret-profitability-function}]
Here we have
\begin{align*}
\totdur = \min\{n\in\N\,\vert\, \sum_{i=1}^n \exprv_i + \rv_i\one\left(\rv_i\geq \thresh_i\right)> \timeT\}.
\end{align*}
\cref{alg:monotone} yields the regret
\begin{align*}
\reg(\timeT) 
&= \profopt \timeT - \E\left[\sum_{i=1}^{\totdur} \rew(\rv_i)\one(\rv_i\geq \thresh_i)\right]\\
&= \profopt\timeT - \E\left[\sum_{i=1}^{\totdur} \E\left[\rew(\rv_i)\one(\rv_i\geq \thresh_i)\right]\right] \text{ by Wald's equation}\\
&= \profopt\timeT - \E\left[\sum_{i=1}^{\totdur} \profset(\thresh_i)\E\left[\rv_i\one(\rv_i\geq \thresh_i) + \exprv_i\right]\right]\\
&\leq \E\left[\sum_{i=1}^{\totdur} (\profopt-\profset(\thresh_i))\E\left[\rv_i\one(\rv_i\geq \thresh_i)+ \exprv_i\right]\right].
\end{align*}
Note that conversely to the previous sections, when using Wald's equation the expectation here is taken conditionally to threshold $\thresh_i$. 
Bounding separately the first two terms and the whole sum for small probability events, \cref{prop:regret-monotone} yields
\begin{align*}
\reg(\timeT) 
&\leq \sum_{n=3}^{\numstage} \left( 4\confprofset_{n-1} + \frac{\kappa_7}{n} \right) + 2(\ubrew-\lbrew) + (\ubrew-\lbrew)(\numstage\delta + \E\left[(\totdur - S)_+\right]) \\
&\leq 8\left(\sqrt{\sigma^2 + \frac{(\ubrew - \lbrew)^2}{4}} + \frac{\ubrew-\lbrew}{\sqrt{2}}(\pr\ubrv+2)\right)\sqrt{S\ln\left(\frac{2(\numstage+1)}{\delta}\right)} \\
& \phantom{\leq} + (4\pr(\ubrew-\lbrew) + \kappa_7) \ln(e\numstage) +(\ubrew-\lbrew)(2+\numstage\delta + \E\left[(\totdur - S)_+\right])
\end{align*}
Using the given values for $S, \delta$ and \cref{lem:poisson} finally yields \cref{th:regret-profitability-function}.
\end{proof}

\subsection{Proofs of \cref{subsec:lower}}\label{app:lowerproof}

\begin{proof}[Proof of \cref{th:lower}] \hfill \\
1) We first prove for all algorithms in \cref{subsec:bandit,,sec:fast}, when the reward function is unknown and noisy observations are observed.

Following the same lines of the proof of \cref{th:regret-bandit}, we first write the regret in a bandit form:
\begin{align} \label{eq:lower1}
\reg(\timeT) & \geq \E\left[\sum_{n=1}^\totdur \E\Big[\big(\rew(\rv_n)-\profopt \rv_n\big)\big(\one(\rew(\rv_n)\geq\profopt \rv_n)-\one(\refuse(n))\big)\Big]\right] - \profopt \ubrv.
\end{align}
Now consider the one arm contextual bandit where, given a context $X$, the normalized reward is $\profopt X$ and the arm returns a noised reward of mean $\rew(X)$. The sum in \cref{eq:lower1} exactly is the regret incurred by the strategy $A$ in this one armed contextual bandit problem, with horizon $\totdur$.

Although $\totdur$ is random and depends on the strategy of the agent, it is larger than $\frac{\pr \timeT}{1+\ubrv}$ with probability at least $\alpha>0$, constant in $T$. Moreover, the one armed contextual bandit problem is easier than the agent problem, as $\profopt$ is known only in the former. Thanks to these two points, $\reg(\timeT)$ is larger than $$\reg(\timeT) \geq \alpha\tilde{\reg}\left(\frac{\pr \timeT}{1+\ubrv}\right)- \profopt \ubrv$$ where $\tilde{\reg}$ is the regret in this one armed contextual bandit problem. 
\cref{th:lower} then follows from classical results in contextual bandits \citep[see, \eg][]{AudTsy:AS2007,rigollet2010nonparametric}.

Note here that we only consider a subclass of all the one armed contextual bandit
problems. Indeed, we fixed the normalized reward to $\profopt \rv$ and the arm reward must satisfy $\profopt = \pr \E[(\rew(\rv)-\profopt\rv)_+]$.
Yet this subclass is large enough and the existing proofs \citep{AudTsy:AS2007,rigollet2010nonparametric} can be easily adapted to this setup.

\medskip

2) It now remains to prove a $\Omega(\sqrt{T})$ bound when the reward function $\rew$ is known. Consider the following setting
\begin{gather*}
\rv = \begin{cases} 1 \text{ with proba } \frac{1}{2} + \eps \\ 
2 \text{ with proba } \frac{1}{2} - \eps\end{cases} \\
\text{and} \quad \begin{cases} r(1)=\frac{1}{2} \\ r(2) = 2.\end{cases}
\end{gather*}
Now consider two worlds where $\eps = \Delta$ in the first one, and $\eps = -\Delta$ in the second one for some $\Delta > 0$. Basic calculations then give the following
\begin{gather*}
\profopt_1 = \frac{1}{2} - \frac{2 \Delta}{5} + o(\Delta) \quad \text{in the first world},\\
\profopt_2 = \frac{1}{2} + \frac{\Delta}{2} + o(\Delta) \quad \text{in the second world}.
\end{gather*}
The optimal strategy then accepts all tasks in the first world, while it only accepts tasks of duration $2$ in the second one. Classical lower bound techniques \citep[see, \eg][Theorem 14.2]{lattimore2020} then show that for $\Delta = \frac{1}{\sqrt{\pr \timeT}}$, with some constant probability and positive constant $\alpha$ both independent from $\timeT$, any strategy
\begin{itemize}
\item either rejects $\alpha \pr \timeT$ tasks $X_t = 1$ in the first world,
\item or accepts $\alpha \pr \timeT$ tasks $X_t = 1$ in the second world after receiving $\pr \timeT$ task propositions.
\end{itemize}

In any case, this means that any strategy has a cumulative regret of order $\sqrt{\timeT}$ in at least one world.\end{proof}
\end{document}